\documentclass[letterpaper,11pt]{article}
\usepackage[margin=1in]{geometry}
\usepackage[bookmarks, colorlinks=true, plainpages = false, colorlinks=true,
citecolor=blue,
linkcolor=blue,
anchorcolor=red,
urlcolor=blue]{hyperref}
\usepackage{url}
\usepackage{amsmath,amssymb,amsfonts,amsthm}
\usepackage{dsfont}
\usepackage{color}

\usepackage[utf8]{inputenc} 
\usepackage[T1]{fontenc}    
\usepackage{hyperref}       
\usepackage{url}            
\usepackage{booktabs}       
\usepackage{amsfonts}       
\usepackage{nicefrac}       
\usepackage{microtype}      
\usepackage{xcolor}         
\usepackage[round]{natbib}
\usepackage{algorithm}
\usepackage{algpseudocode}
\usepackage{graphicx}
\usepackage{subfigure}
\usepackage{dsfont}
\usepackage{enumitem}
\usepackage{makecell}
\usepackage{caption}
\usepackage{multirow}
\usepackage{wrapfig}

\DeclareMathOperator*{\argmax}{arg\,max}

\newtheorem{assumption}{Assumption}
\newtheorem{theorem}{Theorem}
\newtheorem{lemma}{Lemma}
\newtheorem{corollary}{Corollary}
\newtheorem{definition}{Definition}
\newtheorem{proposition}{Proposition}

\newcommand{\Eb}{\mathbb{E}}
\newcommand{\Pb}{\mathbb{P}}

\newcommand{\Sc}{\mathcal{S}}
\newcommand{\Fc}{\mathcal{F}}
\newcommand{\Ec}{\mathcal{E}}
\newcommand{\Gc}{\mathcal{G}}
\allowdisplaybreaks

\ifodd 1

\else

\fi

\ifodd 1
\newcommand{\congc}[1]{{\color{red}(Cong: #1)}}
\else
\newcommand{\congc}[1]{}
\fi

\ifodd 0
\newcommand{\shir}[1]{{\color{blue}#1}}
\else
\newcommand{\shir}[1]{#1}
\fi

\ifodd 0

\else

\fi

\setlist[itemize]{noitemsep, topsep=0pt}

\title{Heterogeneous Multi-player Multi-armed Bandits: \\Closing the Gap and Generalization\footnotetext{To appear at the 35th Conference on Neural Information Processing Systems (NeurIPS), 2021. }}

\author{
	Chengshuai Shi\thanks{Department of Electrical and Computer Engineering, University of Virginia, Charlottesville, VA, USA; Email:  \texttt{\{cs7ync,cong@virginia.edu\}}.}
	\qquad\qquad Wei Xiong\thanks{Department of Mathematics, The Hong Kong University of Science and Technology, Hong Kong, China; Email:  \texttt{wxiongae@connect.ust.hk}.}
	\qquad\qquad Cong Shen\footnotemark[1]
	\qquad\qquad Jing Yang\thanks{School of Electrical Engineering and Computer Science, The Pennsylvania State University, University Park, PA, USA; Email: \texttt{yangjing@psu.edu}.} 
}
\begin{document}

\maketitle

\begin{abstract}
Despite the significant interests and many progresses in decentralized multi-player multi-armed bandits (MP-MAB) problems in recent years, the regret gap to the natural centralized lower bound in the heterogeneous MP-MAB setting remains open. In this paper, we propose BEACON -- \emph{Batched Exploration with Adaptive COmmunicatioN} -- that closes this gap. BEACON accomplishes this goal with novel contributions in implicit communication and efficient exploration. For the former, we propose a novel adaptive differential communication (ADC) design that significantly improves the implicit communication efficiency. For the latter, a carefully crafted batched exploration scheme is developed to enable incorporation of the combinatorial upper confidence bound (CUCB) principle. We then generalize the existing \emph{linear-reward} MP-MAB problems, where the system reward is always the sum of individually collected rewards, to a new MP-MAB problem where the system reward is a general (nonlinear) function of individual rewards. We extend BEACON to solve this problem and prove a \shir{logarithmic} regret. BEACON bridges the algorithm design and regret analysis of combinatorial MAB (CMAB) and MP-MAB, two largely disjointed areas in MAB, and the results in this paper suggest that this previously ignored connection is worth further investigation.

\end{abstract}

\section{Introduction}\label{sec:intro}
Motivated by the application of cognitive radio \citep{anandkumar2010opportunistic,anandkumar2011distributed,gai2010learning}, the multi-player version of the multi-armed bandits problem (MP-MAB) has sparked significant interests in recent years. MP-MAB takes player interactions into account by having multiple decentralized players simultaneously play the bandit game and interact with each other through arm collisions.

Prior MP-MAB studies mostly focus on the homogeneous variant, where the bandit model is assumed to be the same across players \citep{liu2010distributed,rosenski2016multi,besson2018multi}. Recent attentions have shifted towards the more general MP-MAB model with player-dependent bandit instances (i.e., the {\em heterogeneous} variant) \citep{bistritz2018distributed,bistritz2018game,tibrewal2019multiplayer,boursier2019practical}. However, unlike the homogeneous variant, the current understanding of the heterogeneous setting is still limited.
\begin{itemize}[leftmargin=*]\itemsep=0pt
    \item Recent advances \citep{boursier2019sic,proutiere2019optimal} show that for the homogeneous setting, decentralized MP-MAB algorithms can achieve almost the same performance as centralized ones. However, state-of-the-art results in heterogeneous variants still have significant gaps from the centralized performance. It remains an open problem whether a decentralized algorithm can approach the centralized performance for the heterogeneous MP-MAB variant. 
    
    \item All prior MP-MAB works are confined to a {\em linear} system reward function: the system reward is the sum of individual outcomes from players. However, practical system objectives are often captured by more complicated nonlinear reward functions, e.g., the minimal function (see Section \ref{subsec:reward}). 
\end{itemize}

In this paper, we make progress in the aforementioned problems for decentralized heterogeneous MP-MAB. A novel algorithm called BEACON -- Batched Exploration with Adaptive COmmunicatioN -- is proposed and analyzed. In particular, this work makes the following contributions.

\begin{itemize}[leftmargin=*]\itemsep=0pt
    \item BEACON introduces several novel ideas to the design of implicit communication and efficient exploration. For the former, a novel adaptive differential implicit communication (ADC) scheme is proposed, which can significantly lower the implicit communication loss compared to the state of the art. For the latter, core principles from CUCB  \citep{chen2013combinatorial} are incorporated with a batched exploration design, which leads to both efficient and effective explorations. 
    \item For the linear reward function, we rigorously show that regret bounds of BEACON, both problem-dependent and problem-independent, not only improve all prior regret analyses but more importantly are capable of \emph{approaching the centralized lower bounds}, thus answering the aforementioned open problem positively.
    \item We then propose to generalize the study of heterogeneous MP-MAB to general (nonlinear) reward functions. BEACON is extended to solve such problems and we show that it achieves a regret of $O(\log(T))$, \shir{where $T$ is the time horizon}. The analysis itself holds important value as it bridges the regret analysis of combinatorial MAB (CMAB) and MP-MAB.

    \item BEACON achieves impressive empirical results. It not only outperforms existing decentralized algorithms significantly, but indeed has a comparable performance as the centralized benchmark, hence corroborating the theoretical analysis. Remarkably, BEACON with the linear reward function generally achieves $\sim 6\times$ improvement over the state-of-the-art METC \citep{boursier2019practical}.
\end{itemize}

\section{Problem Formulation}\label{sec:problem}
A decentralized MP-MAB model consists of $K$ arms and $M$ players. As commonly assumed \citep{bistritz2018game,boursier2019practical}, there are more arms than players, i.e., $M\leq K$, and initially the players have knowledge of $K$ but not $M$. Furthermore, no \emph{explicit} communications are allowed among players, which results in a decentralized system. Also, time is assumed to be slotted, and at time step $t$, each player $m\in [M]$ chooses and pulls an arm $s_m(t)\in[K]$. The action vector of all players at time $t$ is denoted as  $S(t):=\left[s_1(t),...,s_M(t)\right]$, \shir{which is referred to as a ``matching'' for convenience although it is not necessarily one-to-one.}

\textbf{Individual Outcomes.} For each player $m$, an outcome\footnote{The term ``outcome'' distinguishes players' individual rewards from the later introduced system rewards.} $O_{k,m}(t)$ is associated with her action of pulling arm $s_m(t)=k$ at time $t$, which is defined as
\begin{equation}\label{eqn:obs}
		O_{k,m}(t):=X_{k,m}(t)\eta_{k}(S(t)).
\end{equation}
In Eqn.~\eqref{eqn:obs}, $X_{k,m}(t)$ is a random variable of arm utility and $\eta_{k}(S(t))$ is the no-collision indicator defined by $\eta_{k}(S):=\mathds{1}\{ | C_k(S) |\leq 1\}$ with $C_k(S):=\{n\in[M]|s_n=k\}$. 
In other words, if player $m$ is the only player choosing arm $k$, the outcome is $X_{k,m}(t)$; if multiple players choose arm $k$ simultaneously, a collision happens on this arm and the outcome is zero regardless of $X_{k,m}(t)$.

For a certain arm-player pair, i.e., $(k,m)$, the set of random arm utilities $\{X_{k,m}(t)\}_{t\geq 1}$ is assumed to be sampled independently from an unknown distribution $\phi_{k,m}$, which has a bounded support on $[0,1]$ and an unknown expectation $\Eb[X_{k,m}(t)]=\mu_{k,m}$. In general, these utility distributions are player-dependent, i.e., $\mu_{k,m}\neq \mu_{k,n}$ when $m\neq n$. Note that despite the time independence among  $\{X_{k,m}(t)\}_{t\geq 1}$ for a certain arm-player pair $(k,m)$, correlations can exist among the random utility variables of different arm-player pairs, i.e., among $X_{k,m}(t)$ for different $(k,m)$ pairs.

To ease the exposition, we define $\Sc=\left\{S=[s_1,...,s_M]|s_m\in[K], \forall m\in[M]\right\}$ as the set of all possible matchings $S$ and abbreviate the arm $k$ of player $m$ as arm $(k,m)$. We further denote $\boldsymbol{\mu} = \left[\mu_{k,m}\right]_{(k,m)\in[K]\times[M]}$ and $\boldsymbol{\mu}_S = \left[\mu_{s_m,m}\right]_{m\in[M]}$ for $S=[s_1,...,s_M]$.

\textbf{System Rewards.} Besides players' individual outcomes, with matching $S(t)$ chosen at time $t$, a random system reward, denoted as $V(S(t),t)$, is collected for the entire system. The most commonly-studied reward function \citep{bistritz2018game,boursier2019practical} is the sum of outcomes from different players (referrd to as the \emph{linear reward function}), i.e., $
    V(S(t),t) := \sum\nolimits_{m\in [M]} O_{s_m(t),m}(t)$.
With this linear reward function, for matching $S$, the expected system reward is denoted as $V_{\boldsymbol{\mu},S}: = \Eb[V(S,t)] = \sum_{m\in[M]}\mu_{s_m,m}\eta_{s_m}(S)$ under matrix $\boldsymbol{\mu}$. As almost all of the existing MP-MAB literature focus on the linear reward function, we also focus on this case first, but note that the problem formulation presented in this section can be extended to general (nonlinear) reward functions in Section~\ref{sec:general}.

\textbf{Feedback Model.} Different feedback models exist in the MP-MAB literature, and this work focuses on the collision-sensing model \citep{bistritz2018distributed,bistritz2018game,boursier2019sic,boursier2019practical}. Specifically, player $m$ can access her own outcome $O_{s_m(t),m}(t)$ and the corresponding no-collision indicator $\eta_{s_m(t)}(S(t))$, but neither the overall reward $V(S(t),t)$ nor outcomes of other players. In other words, at time $t$, player $m$ chooses arm $s_m(t)$ based on her own history $H_{m}(t) = \left\{s_m(\tau),O_{s_m(\tau),m}(\tau), \eta_{s_m(\tau)}(S(\tau))\right\}_{1\leq \tau\leq t-1}$.

\textbf{Regret Definition.} If $\boldsymbol{\mu}$ is known \textit{a priori}, the optimal choice is the matching that gives the highest expected reward  $V_{\boldsymbol{\mu},*}:=\max_{S\in\Sc}V_{\boldsymbol{\mu},S}$. We formally define the regret after $T$ rounds of playing as
\begin{equation}\label{eqn:regret_def}
    R(T) = TV_{\boldsymbol{\mu},*}-\Eb\left[\sum_{t=1}^T V(S(t),t)\right],
\end{equation}
where the expectation is w.r.t. the randomness of the policy and the environment.

One technical novelty worth noting is that this work considers the general case with possibly {\em multiple optimal matchings}, instead of the commonly assumed unique one \citep{bistritz2018distributed,bistritz2018game}. Multiple optimal matchings might be uncommon for the linear reward function, but often occur under more sophisticated reward functions that will be discussed later, e.g., the minimal function, and brings substantial difficulties into player coordination. In addition, the proposed BEACON design is also applicable to the homogeneous setting, i.e., $\forall m \in [M],\mu_{k,m}=\mu_k$, with some trivial adjustments.

\section{The BEACON Algorithm}\label{sec:alg}
\subsection{Algorithm Structure and Key Ideas}
After the orthogonalization procedure \citep{proutiere2019optimal} at the beginning of the game, during which each player individually estimates the number of players $M$ and assigns herself of a unique index $m\in[M]$, BEACON proceeds in epochs and each epoch consists of two phases: (implicit) communication and exploration.\footnote{Details of the orthogonalization procedure are given in Appendix~\ref{sub_app:orthogonal}. In addition, by ``exploration phase'', we mean the time steps in one epoch that are not used for (implicit) communications, which actually contain both exploration and exploitation.} While similar two-phase structures have been adopted by other heterogeneous MP-MAB algorithms \citep{tibrewal2019multiplayer,boursier2019practical}, those designs fail to have regrets approaching the centralized lower bound.

The challenge in approaching the centralized lower bound is not only designing more efficient implicit communications and explorations, but also connecting them in a way that neither phase dominates the overall regret and both approach the centralized lower bound simultaneously. BEACON precisely achieves these goals, with several key ideas that not only are crucial to closing the regret gap but also hold individual values in MP-MAB research.  
First, a novel adaptive differential communication (ADC) method is proposed, which is fundamental in improving the effectiveness and efficiency of implicit communications. Specifically, ADC drastically reduces the communication cost from up to $O(\log(T))$ per epoch in state-of-the-art designs \citep{boursier2019practical} to $O(1)$ per epoch, which ensures a low communication cost. Second, CUCB principles \citep{chen2013combinatorial} are incorporated with a batched exploration structure to ensure a low exploration loss \shir{(see Section~\ref{sec:related} for more discussions on the relationship between CMAB and MP-MAB)}. CUCB principles address a critical challenge of \emph{large amount of matchings} in heterogeneous MP-MAB \shir{(i.e., $|\Sc|=K^M$)}, which hampered prior designs. The batched structure, on the other hand, is carefully embedded and optimized such that the need of communication and exploration is balanced, leading to neither dominating the overall regret.

\subsection{Batched Exploration}\label{subsec:alg_expl}
To facilitate the illustration, we first present the batched exploration scheme and also a sketch of BEACON under an imaginary communication-enabled setting, which will be addressed in Section~\ref{subsec:alg_comm}. Specifically, players are assumed to be able to communicate with each other freely in this subsection. 

The batched exploration proceeds as follows. At the beginning of epoch $r$, each player $m$ maintains an arm counters $p^r_{k,m}$ for each arm $k$ of hers. The counters are updated as $p^r_{k,m}=\lfloor \log_2(T^r_{k,m})\rfloor$, where $T^r_{k,m}$ is the number of exploration pulls on arm $(k,m)$ up to epoch $r$. {Then, the leader (referring to the player with index $1$) collects arm statistics from followers (referring to the players other than the leader). Specifically, if $p^{r}_{k,m}>p^{r-1}_{k,m}$, statistics $\tilde{\mu}^{r}_{k,m}$ is collected from follower $m$; otherwise, $\tilde{\mu}^{r}_{k,m}$ is not updated and kept the same as $\tilde{\mu}^{r-1}_{k,m}$, where $\tilde{\mu}^r_{k,m}$ is a to-be-specified characterization of arm $(k,m)$'s sample mean $\hat{\mu}^r_{k,m}$. With the updated information}, an upper confidence bound (UCB) matrix $\boldsymbol{\bar{\mu}}_r = [\bar{\mu}^r_{k,m}]_{(k,m)\in[K]\times[M]}$ is calculated by the leader, where $\bar{\mu}^r_{k,m}=\tilde{\mu}^r_{k,m}+\sqrt{3\ln t_r/2^{p^r_{k,m}+1}}$, and $t_r$ is {the time step at the beginning of epoch $r$}. 

The UCB matrix $\boldsymbol{\bar{\mu}}_r$ is then fed into a combinatorial optimization solver, denoted as $\texttt{Oracle}(\cdot)$, which outputs the optimal matching w.r.t. the input. Specifically, $S_r = [s_1^r,...,s_M^r]\gets \texttt{Oracle}(\boldsymbol{\bar{\mu}}_r) = \argmax_{S\in \Sc}\left\{\sum_{s_m}\bar{\mu}^r_{s_m,m}\right\}$, which can be computed with a polynomial time complexity using the Hungarian algorithm \citep{munkres1957algorithms}. We note that similar optimization solvers are also required by \citet{boursier2019practical,tibrewal2019multiplayer}.
Inspired by the exploration choice of CUCB, this matching $S_r$ is chosen to be explored. The leader thus assigns the matching $S_r$ to followers, i.e., arm $s_m^r$ for player $m$.

{After the assignment, the exploration begins}. One important ingredient of BEACON is that the duration of exploring the chosen matching, i.e., the adopted batch size, is determined by the {smallest} arm counter in it. Specifically, for $S_r$, we denote $p_r = \min_{m\in[M]}p^r_{s^r_m,m}$ and the batch size is chosen to be $2^{p_r}$. In other words, during the following $2^{p_r}$ time steps, players are fixated to exploring the matching $S_r$. {Then, epoch $r+1$ starts, and the same procedures are iterated.}

\textbf{Remarks.} 
BEACON directly selects the matching with the largest UCB to explore. It turns out that this natural method significantly outperforms the ``matching-elimination'' scheme in \citet{boursier2019practical}, and is critical to achieving a near-optimal exploration loss. In addition, the chosen batch size of $2^{p_r}$ ensures \emph{sufficient but not excessive} pulls w.r.t. the least pulled arm(s) in the chosen matching, which dominate the uncertainties. Furthermore, while similar batched structures have been utilized in the bandit literature \citep{auer2002finite,hillel2013distributed}, the updating of arm counters in BEACON is carefully tailored. Last, the leader collects followers' statistics only when arm counters increase, i.e., $p^{r}_{k,m}>p^{r-1}_{k,m}$, which means $\tilde{\mu}^{r}_{k,m}$ is sufficiently more precise than $\tilde{\mu}^{r-1}_{k,m}$. This design contributes to a low communication frequency while not affecting the exploration efficiency.  

\subsection{Efficient Implicit Communication}\label{subsec:alg_comm}
{Since} explicit communication is prohibited in decentralized MP-MAB problems, we now discuss how to use \emph{implicit} communication \citep{boursier2019sic} to share information in BEACON. Specifically, players can take predetermined turns to ``communicate'' by having the ``receive'' player sample one arm and the ``send'' player either pull (create collision; bit $1$) or not pull (create no collision; bit $0$) the same arm to transmit one-bit information. {Although information sharing is enabled}, such a forced-collision communication approach is inevitably costly, as collisions reduce the rewards. The challenge now is how to keep the communication loss small, ideally $O(\log(T))$.

	\begin{algorithm}[htb]
	    \caption{BEACON: Leader}
	    \label{alg:leader}
	    \begin{algorithmic}[1]
		\State Initialization: $r\gets 0$; $\forall (k,m), p^r_{k,m}\gets -1, T^r_{k,m}\gets 0, \tilde{\mu}^r_{k,m}\gets 0$
		\State Play each arm $k\in[K]$ and $T^{r+1}_{k,1}\gets T^r_{k,1}+1$
		\While{not reaching the time horizon}
		\State $r\gets r+1$
		\State $\forall (k,m), p^r_{k,m}\gets \left\lfloor\log_2(T^r_{k,m})\right\rfloor$
		\State $\forall k\in[K]$, update sample mean $\hat{\mu}^{r}_{k,1}$ with the first $2^{p^r_{k,1}}$ exploratory samples from arm $k$
		\Statex $\triangleright$ \textit{Communication Phase}
		\For{$(k,m)\in[K]\times [M]$} 
		\If{$p^r_{k,m}>p^{r-1}_{k,m}$}
		\State $\tilde{\delta}^r_{k,m}\gets \texttt{Receive}(\tilde{\delta}^r_{k,m},m)$
		\State $\tilde{\mu}_{k,m}^r\gets \tilde{\mu}_{k,m}^{r-1}+\tilde{\delta}_{k,m}^r$
		\Else 
		\State $\tilde{\mu}_{k,m}^r\gets \tilde{\mu}_{k,m}^{r-1}$
		\EndIf
		\EndFor
		\State $\forall (k,m), \bar{\mu}^r_{k,m}\gets \tilde{\mu}_{k,m}^r+\sqrt{3\ln t_r/2^{p^r_{k,m}+1}}$
		\State $S_r=[s^r_1,...,s^r_M]\gets \texttt{Oracle}(\boldsymbol{\bar{\mu}_r})$
		\State $\forall m\in [M], \texttt{Send}(s^r_m,m)$
		\Statex $\triangleright$ \textit{Exploration Phase}
		\State $p_r \gets \min_{m\in[M]}p^r_{s^r_m,m}$
		\State Play arm $s^r_1$ for $2^{p_r}$ times
		\State Signal followers to stop exploration
		\State Update $\forall m\in[M], T^{r+1}_{s_m,m}\gets T^r_{s_m,m}+2^{p_r}$
		\EndWhile
		\end{algorithmic}
	\end{algorithm}

The {\em batched} exploration scheme plays a key role in reducing the communication loss via infrequent information updating. In other words, players only communicate statistics and decisions before each batch instead of each time step. With the aforementioned batch size, there are at most $O(\log(T))$ epochs in horizon $T$. Thus, intuitively, if the communication loss per epoch can be controlled of order $O(1)$ irrelevant of $T$, the overall communication loss would not be dominating. However, this requirement is challenging and none of the existing implicit communication schemes \citep{boursier2019sic,boursier2019practical} can meet it, which calls for a novel communication design.

From the discussion of the exploration phases, we can see that sharing  arm statistics $\tilde{\mu}^r_{k,m}$ is the most challenging part. Specifically, as opposed to sharing integers of arm indices in $S_r$ and the batch size parameter $p_r$, statistics $\tilde{\mu}^r_{k,m}$ is often a decimal while forced-collision is fundamentally a {digital communication} protocol. We thus focus on the communication design for sharing statistics $\tilde{\mu}^r_{k,m}$, and propose the adaptive differential communication (ADC) method as detailed below. Details of sharing $S_r$ and $p_r$ can be found in Appendix~\ref{subsec:supp_comm}.

The first important idea is to let followers \textbf{adaptively} quantize sample means for communication. Specifically, upon communication, the arm statistics $\tilde{\mu}^r_{k,m}$ is not directly set as the collected sample mean $\hat{\mu}^r_{k,m}$. Instead, $\tilde{\mu}^r_{k,m}$ is a quantized version of $\hat{\mu}^r_{k,m}$ using $\lceil 1+ p^r_{k,m}/2\rceil$ bits. Since $\tilde{\mu}^{r}_{k,m}$ is communicated only upon an increase of the arm counter $p^r_{k,m}$, this quantization length is adaptive to the arm counter (or equivalently the arm pulls), and further to the adopted confidence bound in Section~\ref{subsec:alg_expl}, i.e., {$\sqrt{3\ln t_r/2^{p^r_{k,m}+1}}$}. However, this idea alone is not sufficient because $p^r_{k,m}$ is of order {up to} $O(\log(T))$, instead of $O(1)$.

To overcome this obstacle, the second key idea is \textbf{differential} communication, which significantly reduces the redundancies in statistics sharing. Specifically, {follower} $m$ first computes the difference $\tilde{\delta}^{r}_{k,m} = \tilde{\mu}^{r}_{k,m}-\tilde{\mu}^{r-1}_{k,m}$, and then truncates the bit string of $\tilde{\delta}^{r}_{k,m}$ upon the most significant non-zero bit, e.g., $110$ for $000110$. She only communicates this truncated version of $\tilde{\delta}^{r}_{k,m}$ {in the transmission of $\tilde{\mu}^{r}_{k,m}$ to the leader}. The intuition is that $\tilde{\mu}^{r}_{k,m}$ and $\tilde{\mu}^{r-1}_{k,m}$ are both concentrated at $\mu_{k,m}$ with high probabilities, which results in a small $\tilde{\delta}^{r}_{k,m}$. From an information-theoretic perspective, the conditional entropy of  $\tilde{\mu}^{r}_{k,m}$ on $\tilde{\mu}^{r-1}_{k,m}$, i.e., $H(\tilde{\mu}^{r}_{k,m}|\tilde{\mu}^{r-1}_{k,m})$, is often small because they are highly correlated.\footnote{\shir{Note that sharing the truncated version of $\tilde{\delta}^r_{k,m}$ results in another difficulty that its length varies for different player-arm pairs and is unknown to the leader. A specially crafted ``signal-then-communicate'' scheme is designed to tackle this challenge and can be found in the Appendix~\ref{subsec:supp_comm}.}}

As will be clear in the regret analysis, putting these two ideas together results in an effective communication design, i.e, the ADC scheme, whose expected regret is of order $O(1)$ per epoch and $O(\log(T))$ overall. This method itself represents an important improvement over prior implicit communication protocols in MP-MAB, whose loss is typically of order $O(\log(T))$ per epoch and $O(\log^2(T))$ in total with multiple optimal matchings \citep{boursier2019sic,boursier2019practical}. Techniques similar to ADC have been utilized in areas outside of MAB, e.g., wireless communications \citep{goldsmith1998adaptive}, with proven success in practice \citep{goldsmith2005wireless}. 

\begin{figure}[tbh]
	\centering
	\includegraphics[width=\linewidth]{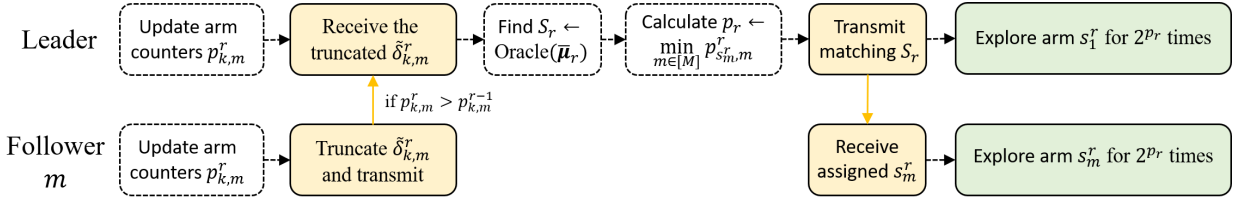}
	\caption{A sketch of epoch $r$ in BEACON. Yellow boxes and yellow lines indicate communications, green boxes for explorations, and boxes with dotted frame for computations.}
	\label{fig:procedure}
\end{figure}

The complete BEACON algorithm can now be obtained by plugging ADC into the batched exploration structure. A sketch of one BEACON epoch is illustrated in Fig.~\ref{fig:procedure}, and the leader's algorithm is presented in Algorithm~\ref{alg:leader}. The follower's algorithm can be found in Appendix~\ref{sub_app:follower}, along with the definitions of the implicit communication protocols denoted by functions \texttt{Send()} and \texttt{Receive()} in Appendix~\ref{sub_app:send_receive}. Note that the for-loops with $(k,m)$ and $\forall (k,m)$ in the pseudo-codes indicate the iteration over all possible arm-player pairs of $[K]\times [M]$. {In addition, the communications of the leader to herself indicated by the pseudo-codes denote her own calculations instead of real forced-collision communications (among the leader and followers), which is a simplification for better exposition.} 

\section{Theoretical Analysis}\label{sec:theory}
With notations $\Sc_c := \{S\in \Sc|\exists m\neq n, s_m=s_n\}$ as the set of collided matchings; $\Sc_{*} := \{S\in \Sc|V_{\boldsymbol{\mu},S} =  V_{\boldsymbol{\mu},*}\}$ as the set of optimal matchings; $\Sc_{b} = \Sc\backslash (\Sc_*\cup \Sc_c)$ as the set of collision-free suboptimal matchings; $\Delta^{k,m}_{\min} := V_{\boldsymbol{\mu},*} - \max\{V_{\boldsymbol{\mu},S}|S\in\Sc_b, s_m = k \}$ \shir{as the minimum sub-optimality gap for collision-free matchings containing arm-player pair $(k,m)$}; $\Delta_{\min} := \min\nolimits_{(k,m)} \{\Delta^{k,m}_{\min}\}$ \shir{as the minimum sub-optimality gap for all collision-free matchings}, 
the regret of BEACON with the linear reward function is analyzed in the following theorem. 
\begin{theorem}\label{thm:linear}
	With the linear reward function, the regret of BEACON is upper bounded as\footnote{With the notation $\tilde{O}(\cdot)$, logarithmic parameters containing $K$ are ignored.}
	\begin{align}
	    R_{\textup{linear}}(T) 
		&= \tilde{O}\left(\sum_{(k,m)\in [K]\times [M]}\frac{M\log(T)}{\Delta^{k,m}_{\min}}+M^2K\log(T)\right)\label{eqn:linear_regret}\\
		&=  \tilde{O}\left(\frac{M^2K\log(T)}{\Delta_{\min}}\right)\notag.
	\end{align}
\end{theorem}
Note that in Eqn.~\eqref{eqn:linear_regret}, the first term represents the exploration regret of BEACON, and the second term the communication regret. Compared with the state-of-the-art regret result $\tilde{O}(M^3K\log(T)/\Delta_{\min})$ for METC \citep{boursier2019practical}, the regret bound in Theorem~\ref{thm:linear} improves the dependence of $M$ from $M^3$ to $M^2$. It turns out that this quadratic dependence is optimal because the same dependence exists in the centralized lower bound (hence a natural lower bound for decentralized MP-MAB) for the linear reward function, as from \citet{kveton2015tight}:\footnote{This lower bound holds for the cases with arbitrarily correlated arms, as considered in this work. Under additional arm independence assumptions \citep{combes2015combinatorial}, lower regrets can be achieved.}
\begin{equation}\label{eqn:linear_lower}
		R_{\textup{linear}}(T) =  \Omega\bigg(\frac{M^2K}{\Delta_{\min}}\log(T)\bigg).
\end{equation}
By comparing Theorem~\ref{thm:linear} and Eqn.~\eqref{eqn:linear_lower}, it can be observed that with the linear reward function, BEACON achieves a regret that approaches the centralized lower bound. The efficiency and effectiveness of both exploration and communication phases are critical in this achievement, as we can see that both terms in Theorem~\ref{thm:linear} are non-dominating at $\tilde{O}(M^2K \log(T))$.

In addition to the problem-dependent bound given in Theorem~\ref{thm:linear}, the following theorem establishes a problem-independent bound, which can be thought of as a worst-case characterization.
\begin{theorem}
\label{thm:linear_gapfree}
    With the linear reward function, it holds that
    \begin{equation*}
        R_{\textup{linear}}(T)= O\left(M\sqrt{KT\log(T)}\right).
    \end{equation*}
\end{theorem}
Theorem~\ref{thm:linear_gapfree} not only improves the best known problem-independent bound $O(M^{\frac{3}{2}}\sqrt{KT\log(T)})$ \citep{boursier2019practical} in the decentralized MP-MAB literature, but also approaches the centralized lower bound $\Omega(M\sqrt{KT})$ \citep{kveton2015tight,merlis2020tight} up to logarithmic factors.

Theorems~\ref{thm:linear} and \ref{thm:linear_gapfree} demonstrate that for the linear reward function, BEACON closes the performance gap (both problem-dependent and problem-independent) between decentralized heterogeneous MP-MAB algorithms and their centralized counterparts. The regret bounds of various MP-MAB algorithms, including BEACON, are summarized in Table~\ref{tbl:regret}.

\shir{\textbf{Remarks.} We note that it is also feasible to combine the ADC protocol and METC \citep{boursier2019practical}, which can address its communication inefficiency, especially with multiple optimal matchings. However, with ideas from CUCB, BEACON is much more efficient in exploration than ``Explore-then-Commit''-type of algorithms (e.g., METC), which is the main reason we did not fully elaborate the combination of METC and ADC in this work. Theoretically, this superiority can be reflected in the extra multiplicative factor  in the exploration loss of METC shown in Table~\ref{tbl:regret}.}

\begin{table*}[thb]
	\caption{Regret Bounds of Decentralized MP-MAB Algorithms}
	\begin{center}
		\begin{tabular}{|c|c|c|c|c|l|}
			\hline
			\multicolumn{1}{|c|}{\multirow{4}*{Algorithm/Reference}} & \multirow{4}*{\makecell[c]{Reward \\ function}}& \multicolumn{3}{c|}{Assumptions} & \multicolumn{1}{c|}{\multirow{4}*{Regret}}\\
			\cline{3-5}
			~ & ~& \makecell[c]{Known \\horizon \\$T$} & \makecell[c]{Known \\gap\\ $\Delta_{\min}$}& \makecell[c]{Unique \\optimal \\matching}& ~\\
			\hline
			\makecell[c]{GoT $\dagger$\\ \citep{bistritz2018game}} & Linear & No & Yes & Yes &$O\left(M\log^{1+\kappa}(T)\right)$\\
			\makecell[c]{Decentralized MUMAB \\ \citep{magesh2019multi}}& Linear & No & Yes & No &$O\left(K^3\log(T)\right)$\\
			\makecell[c]{ESE1\\ \citep{tibrewal2019multiplayer}}& Linear & No & No & Yes &$O\left(\frac{M^2K}{\Delta_{\min}^2}\log(T)\right)$\\
			\makecell[c]{METC\\ \citep{boursier2019practical}} & Linear & Yes & No & Yes &$O\left(\frac{M^3K}{\Delta_{\min}}\log(T)\right)$\\
			\makecell[c]{METC\\ \citep{boursier2019practical}}& Linear & Yes & No & No &$O\left(MK\left(\frac{M^2\log(T)}{\Delta_{\min}}\right)^{1+\iota}\right)$\\
			\makecell[c]{BEACON  \\(this work, Thm.~\ref{thm:general})} & General & No& No & No &$\tilde{O}\left(\frac{MK\Delta_{\max}}{(f^{-1}(\Delta_{\min}))^2}\log(T)\right)$\\
			\makecell[c]{BEACON \\ (this work, Thm.~\ref{thm:linear})} & Linear & No & No & No &$\tilde{O}\left(\frac{M^2K}{\Delta_{\min}}\log(T)\right)$\\
		    \makecell[c]{Lower bound \\\citep{kveton2015tight}} & Linear & N/A & N/A & N/A& $\Omega\left(\frac{M^2K}{\Delta_{\min}}\log(T)\right)$\\
		 	\hline 				
		\end{tabular}
	\end{center}
	\begin{center}
		$\dagger$: tuning parameters in GoT requires knowledge of arm utilities; \\$\kappa, \iota$: arbitrarily small non-zero constants.
	\end{center}
	\label{tbl:regret}
\end{table*} 

\section{Beyond Linear Reward Functions}\label{sec:general}
\subsection{General Reward Functions}\label{subsec:reward}
In this section, we move away from the linear reward functions in almost all prior MP-MAB research, and extend the study to general (nonlinear) reward functions. Two exemplary nonlinear reward functions are given below, with more examples provided in Appendix~\ref{sub_app:example_general}.
\begin{itemize}[leftmargin=*]\itemsep=0pt
	\item Proportional fairness: $V(S,t) = \sum_{m\in[M]}\omega_m\ln(\epsilon+O_{s_m,m}(t))$, where $\epsilon>0$ and $\omega_m>0$ are constants. It promotes fairness among players \citep{mo2000fair};
	\item Minimal: $V(S,t)=\min_{m\in [M]}\{O_{s_m,m}(t)\}$, which indicates the system reward is determined by the least-rewarded player, i.e., the short board of the system;\footnote{Differences with the max-min fairness \citep{bistritz2020my} are elaborated in Appendix~\ref{sub_app:comp_maxmin}.}
\end{itemize}
These reward functions all hold their value in real-world applications, but are largely ignored and cannot be effectively solved by previous approaches. The difficulty introduced by this extension not only lies in the complex mapping from the (unreliable) individual outcomes to system rewards, but also comes from the potential ``coupling'' effect among players (e.g., the minimal reward function). 

To better characterize the problem, the following mild
assumptions are considered.
\begin{assumption}\label{asp:reward}
	There exists an expected reward function $v(\cdot)$ such that $V_{\boldsymbol{\mu},S}:=\Eb[V(S,t)] = v(\boldsymbol{\mu}_S\odot \boldsymbol{\eta}_{S})$, where $\boldsymbol{\eta}_{S} := [\eta_{s_m}(S)]_{m\in [M]}$ and $\boldsymbol{\mu}_S\odot \boldsymbol{\eta}_{S}:= [\mu_{s_m,m}\eta_{s_m}(S)]_{m\in[M]}$.
\end{assumption}
\begin{assumption}[Monotonicity]\label{asp:mono}
	The expected reward function is monotonically non-decreasing with respect to the vector $\boldsymbol{\Lambda}=\boldsymbol{\mu}_S\odot \boldsymbol{\eta}_S$, i.e., if $\boldsymbol{\Lambda}\preceq \boldsymbol{\Lambda}'$, we have $v(\boldsymbol{\Lambda})\leq v(\boldsymbol{\Lambda}')$.
\end{assumption}
\begin{assumption}[Bounded smoothness]\label{asp:bound}
	There exists a strictly increasing (and thus invertible) function $f(\cdot)$ such that $\forall \boldsymbol{\Lambda}, \boldsymbol{\Lambda'},|v(\boldsymbol{\Lambda})- v(\boldsymbol{\Lambda}')|\leq f(\|\boldsymbol{\Lambda}-\boldsymbol{\Lambda}'\|_{\infty})$.
\end{assumption}
Assumption~\ref{asp:reward} indicates that the expected reward $V_{\boldsymbol{\mu},S}$ of matching $S$ is determined only by its expected individual outcomes. It is true for the linear reward function, and also generally holds if distributions $\{\phi_{k,m}\}$ are mutually independent and determined by their expectations $\{\mu_{k,m}\}$, e.g., Bernoulli distribution.  Assumptions~\ref{asp:mono} and \ref{asp:bound} concern the monotonicity and smoothness of the expected reward function, which are natural for most practical reward functions, including the above examples. Similar assumptions have been adopted by \citet{chen2013combinatorial,chen2016combinatorial,wang2018thompson}. 

\subsection{BEACON Adaption and Performance Analysis}\label{subsec:theory_general}
In Section~\ref{subsec:alg_expl}, a combinatorial optimization solver $\texttt{Oracle}(\cdot)$ is implemented for the linear reward function. With ideas from CUCB \citep{chen2013combinatorial}, BEACON can be extended to handle a general reward function with a corresponding solver $\texttt{Oracle}(\cdot)$ that outputs the optimal (non-collision) matching w.r.t. the input matrix $\boldsymbol{\mu}'$, i.e., $S'\gets \texttt{Oracle}(\boldsymbol{\mu}')=\argmax_{S\in \Sc\backslash \Sc_c}V_{\boldsymbol{\mu}',S}$.

With such an oracle, the following theorem provides performance guarantees of BEACON.
\begin{theorem}[\textbf{General reward function}]\label{thm:general}
	Under Assumptions~\ref{asp:reward}, \ref{asp:mono}, and  \ref{asp:bound}, denoting $\Delta^{k,m}_{\max} := V_{\boldsymbol{\mu},*} - \min\{V_{\boldsymbol{\mu},S}|S\in\Sc_b, s_m = k \}$ and $\Delta_c := f(1)$, the regret of BEACON is upper bounded as
	\begin{align*}
	    R(T) &= \tilde{O}\left( \sum_{(k,m)\in[K]\times [M]}\left[\frac{\Delta^{k,m}_{\min}}{(f^{-1}(\Delta^{k,m}_{\min}))^2}+ \int_{\Delta^{k,m}_{\min}}^{\Delta^{k,m}_{\max}} \frac{1}{(f^{-1}(x))^2}\mathrm{d}x\right]\log(T)+ M^2K\Delta_c\log(T) \right)\\
		& = \tilde{O}\left(\sum_{(k,m)\in[K]\times [M]}\frac{\Delta^{k,m}_{\max}\log(T)}{(f^{-1}(\Delta^{k,m}_{\min}))^2}+M^2K\Delta_c\log(T)\right).
	\end{align*}
\end{theorem}

With a stronger smoothness assumption, we can obtain a clearer exposition of the regret.
\begin{corollary}\label{col:general}
	Under Assumptions~\ref{asp:reward} and \ref{asp:mono}, if there exists $B>0$ such that $\forall \boldsymbol{\Lambda},\boldsymbol{\Lambda}', |v(\boldsymbol{\Lambda})- v(\boldsymbol{\Lambda}')|\leq B\|\boldsymbol{\Lambda}-\boldsymbol{\Lambda}'\|_{\infty}$, it holds that 
	\begin{equation*}
	    R(T) =  \tilde{O}\left(\sum_{(k,m)\in [K]\times [M]}\frac{B^2}{\Delta^{k,m}_{\min}}\log(T)+M^2KB\log(T)\right).
	\end{equation*}
\end{corollary}

In addition, since the combinatorial optimization problems with general reward functions can be NP-hard, it is more practical to adopt approximate solvers rather than the exact ones \citep{vazirani2013approximation}. To accommodate such needs, we introduce the following definition of $(\alpha,\beta)$-approximation oracle for $\alpha,\beta \in [0,1]$ as in  \citet{chen2013combinatorial,chen2016combinatorial2, chen2016combinatorial,wang2017improving}:
\begin{definition}
    With a matrix $\boldsymbol{\mu}'=[\mu'_{k,m}]_{(k,m)\in[K]\times [M]}$ as input, an $(\alpha,\beta)$-approximation oracle outputs a matching $S'$, such that $\Pb[V_{\boldsymbol{\mu}',S'}\geq \alpha\cdot V_{\boldsymbol{\mu}',*}]\geq \beta$, where $V_{\boldsymbol{\mu}',*}=\max_{S\in \Sc}V_{\boldsymbol{\mu}',S}$.
\end{definition}
With only an approximate solver, it is no longer fair to compare the performance against the optimal reward. Instead, as in the CMAB literature \citep{chen2013combinatorial,chen2016combinatorial2, chen2016combinatorial,wang2017improving}, an $(\alpha,\beta)$-approximation regret is considered: $R_{\alpha,\beta}(T) = T\alpha\beta V_{\boldsymbol{\mu},*}-\Eb[\sum\nolimits_{t=1}^T V(S(t),t)]$,
where the performance is compared to the $\alpha\beta$ fraction of the optimal reward. As shown in Appendix~\ref{app:alpha_beta}, for this $(\alpha,\beta)$-approximation regret, an upper bound similar to Theorem~\ref{thm:general} can be obtained.

\section{Experiments}\label{sec:exp}
In this section, BEACON is empirically evaluated with both linear and general (nonlinear) reward functions. {All results are averaged over $100$ experiments and the utilities follow mutually independent Bernoulli distributions. Additional experimental details, empirical algorithm enhancements and more experimental results (e.g., with a large game), can be found in Appendix~\ref{app:exp}.} 

\textbf{Linear Reward Function.} 
BEACON is evaluated along with the centralized CUCB \citep{chen2013combinatorial} and the state-of-the-art decentralized algorithm METC \citep{boursier2019practical}. 
The decentralized GoT algorithm \citep{bistritz2018game} is also evaluated but its regrets are over $100\times$ larger than those of BEACON, and thus is omitted in the plots. 
Fig.~\ref{fig:linear} reports results under the same instance in \citet{boursier2019practical} with $K=5, M=5$. Although this is a relatively hard instance with multiple optimal matchings and small sub-optimality gaps, BEACON still achieves a comparable performance as CUCB, and significantly outperforms METC: an approximate $7 \times $ regret reduction at the horizon. 

To validate whether this significant gain of BEACON over METC is representative, we plot in Fig.~\ref{fig:linear_random} the histogram of regrets with $100$ randomly generated instances still with $M = 5, K =5, T = 10^6$. Expected arm utilities are uniformly sampled from $[0,1]$ in each instance. It can be observed that the gain of BEACON is very robust -- its average regret is approximately $6\times$ lower than METC.

\begin{figure*}[t]
    \setlength{\abovecaptionskip}{-1pt}
	\centering
	\subfigure[Linear, cumul. regret.]{ \includegraphics[width=0.4\linewidth]{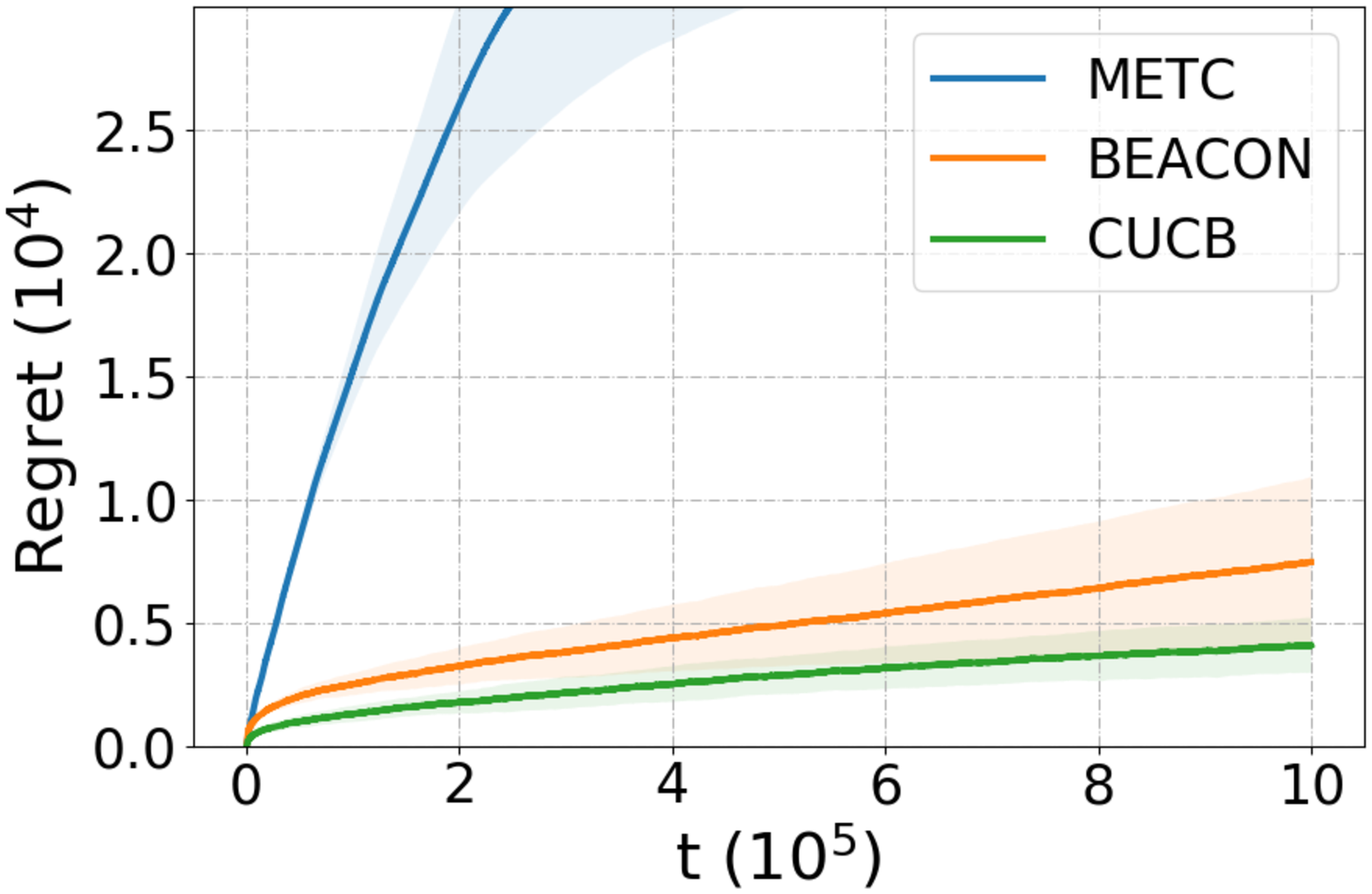}\label{fig:linear}}
	\subfigure[Linear, regret histo.]{ \includegraphics[width=0.4\linewidth]{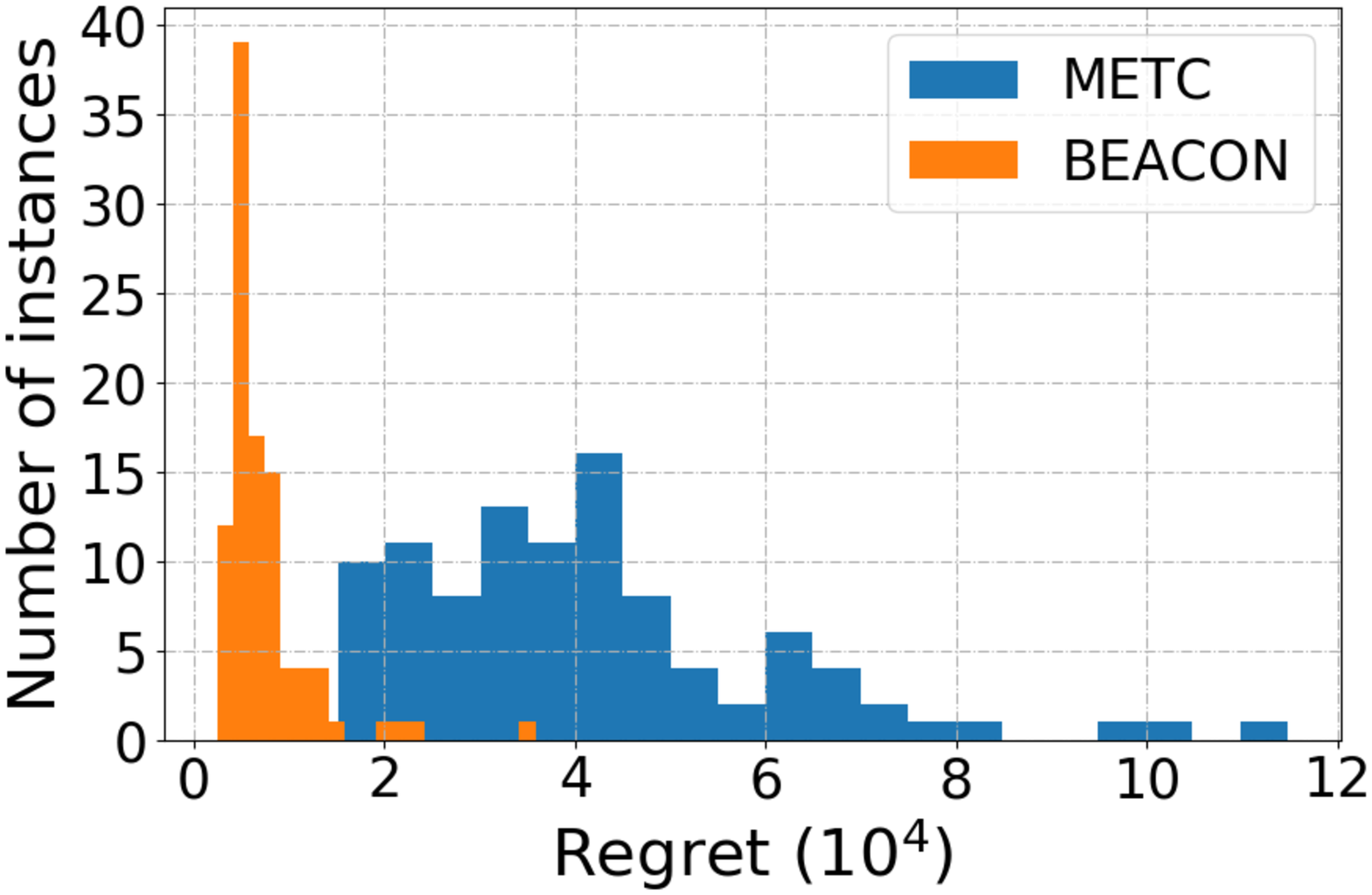}\label{fig:linear_random}}\\
	\subfigure[Proportional fairness.]{ \includegraphics[width=0.4\linewidth]{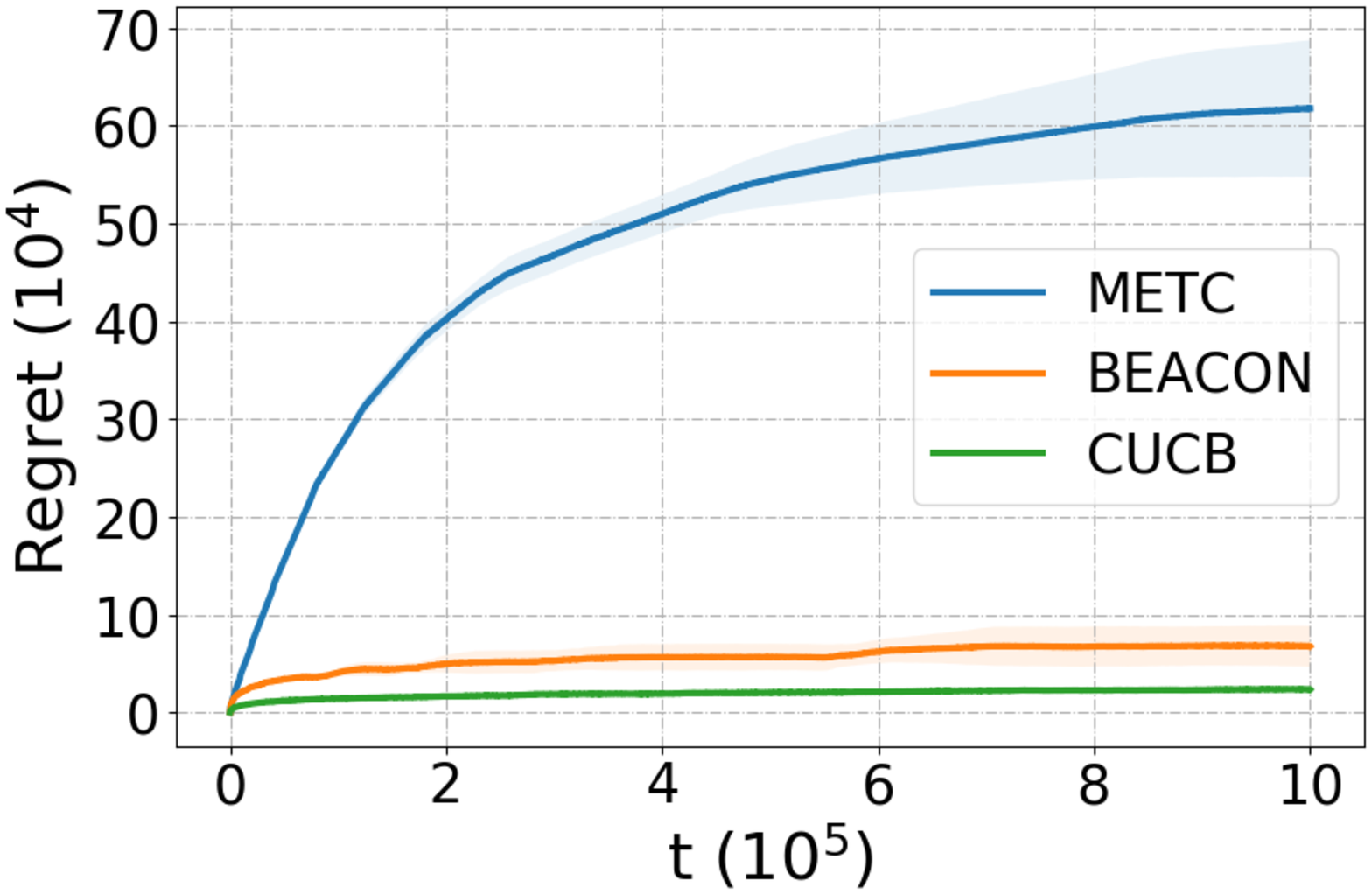}\label{fig:propotional}}
	\subfigure[Minimal.]{ \includegraphics[width=0.4\linewidth]{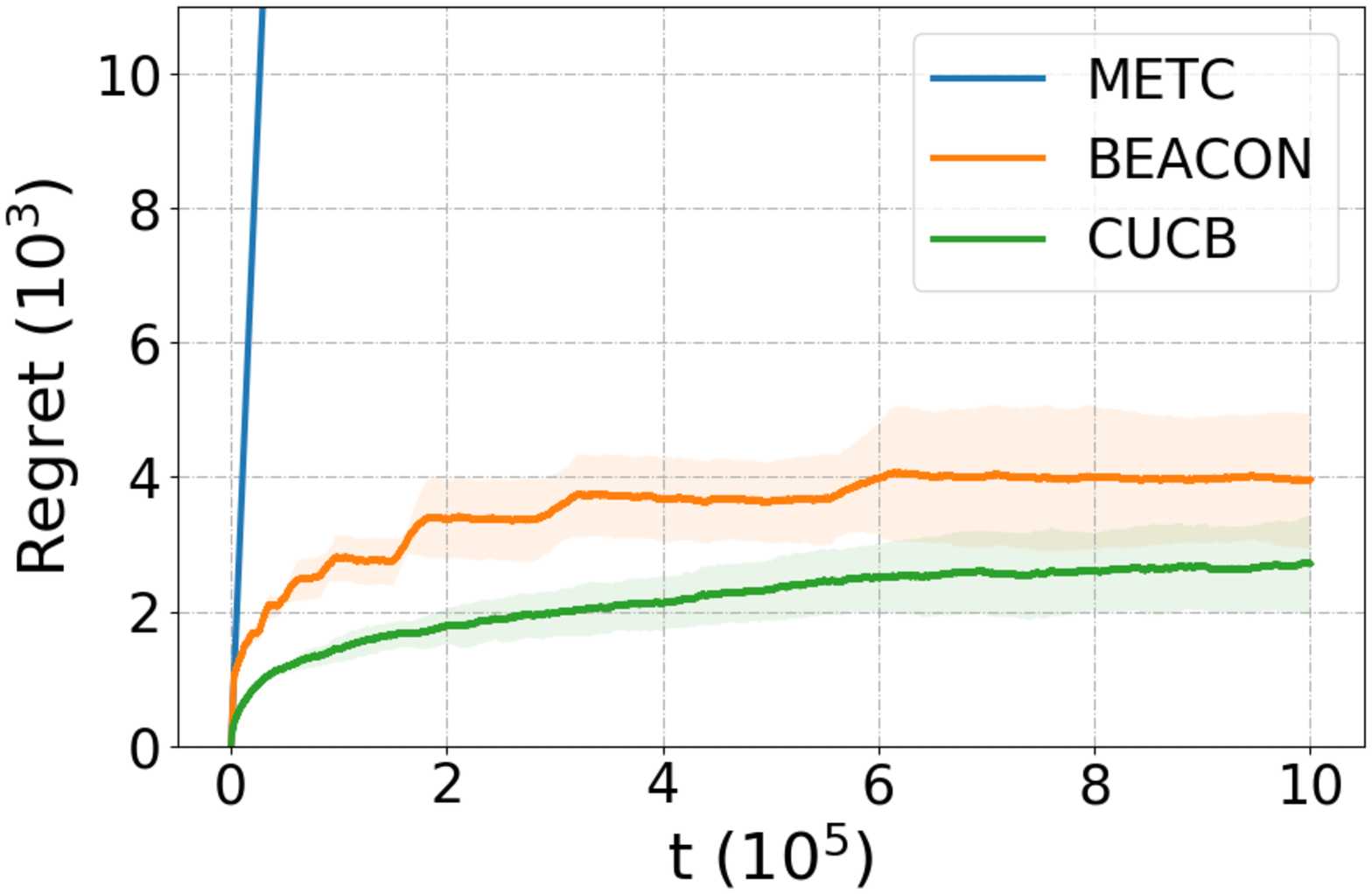}\label{fig:shortboard}}
	\caption{Regret comparisons. The continuous curves represent the empirical average values, and the shadowed areas represent the standard deviations.
	(a), (c) and (d) are evaluated with specific game instances, and (b) is the regret histogram of $100$ randomly generated instances.
	}
	\label{fig:performance}
\end{figure*}

\textbf{General Reward Function.} 
Two representative nonlinear reward functions are used to evaluate BEACON: (1) the proportional fairness function with $\forall m\in [M], \omega_m = 1, \epsilon = 10^{-2}$; (2) the minimal function. BEACON is compared with CUCB and METC.\footnote{To make meaningful comparisons, non-trivial adjustments and enhancements have been applied to METC, which originally applies only to the linear reward function. Details are given in Appendix~\ref{sub_app:metc_enhance}.} Under a game instance with $M = 6, K = 8$, Fig.~\ref{fig:propotional} reports the regrets under the proportional fairness function, and Fig.~\ref{fig:shortboard} with the minimal function. From both results, it can be observed that BEACON has slightly larger (but comparable) regrets than the centralized CUCB, while significantly outperforming METC.

To summarize, BEACON not only significantly outperforms state-of-the-art decentralized MP-MAB algorithms, but is also capable of {\em empirically} approaching the centralized performance, which is the first time for a decentralized heterogeneous MP-MAB algorithm to the best of our knowledge.

\section{Discussions}\label{sec:diss}
We briefly summarize the novel theoretical contributions of this work:
\begin{itemize}[leftmargin=*]\itemsep=0pt
	\item \textbf{Closing the regret gap.} With the linear reward function, BEACON can approach (both problem-dependent and problem-independent) centralized lower bounds. To the best of our knowledge, this is the first time such performance gap is closed (scaling wise) for the heterogeneous MP-MAB.
	\item \textbf{Broader applicability.} BEACON can handle a broad range of general reward functions with a regret of $O(\log(T))$, while existing algorithms mostly focus on the linear reward function and their analyses do not apply to the general reward functions. To the best of our knowledge, this is the first time general reward functions are studied in decentralized MP-MAB.
	\item \textbf{Fewer assumptions.} BEACON achieves a strictly $O(\log(T))$-regret without any assumptions or prior knowledge of the game instance, while prior MP-MAB algorithms typically rely on additional assumptions or knowledge; see Table~\ref{tbl:regret} for details.
\end{itemize}

In addition to these tangible contributions, this work also demonstrates the benefit of incorporating CMAB techniques in the study of MP-MAB. In this paper, both the BEACON design and its regret analysis benefit from CMAB, especially CUCB \citep{chen2013combinatorial,kveton2015tight}. While these two sub-fields of MAB are largely considered disjoint, this work shows that the underlying connection is rather fundamental. This revelation may open up interesting future research directions. For example, under the structure of BEACON, it is conceivable to introduce more advanced CMAB algorithms, e.g., ESCB \citep{combes2015combinatorial}, into the study of MP-MAB with additional assumptions on the arm dependence. In another direction, ideas from this work may also contribute to the study of CMAB. For example, due to the batched structure, BEACON only accesses the oracle $O(\log(T))$ times over $T$ steps, which is more computational efficient than the $O(T)$ times access in CUCB.

Besides contributions, there are open questions left for future studies. First, BEACON relies on a centralized combinatorial optimization solver, i.e., $\texttt{Oracle}(\cdot)$, and so do \citet{boursier2019practical,tibrewal2019multiplayer}. While being a reasonable requirement, this oracle might be computational-infeasible for some applications, e.g., with Internet-of-Things (IoT) devices, especially when $M$ and $K$ are large. Also, while the oracle allows general analysis, it also decouples the problem into two disconnected parts: combinatorial optimization and bandits. It might be helpful to tailor the algorithm into one specific reward function, where joint designs over these two parts can be performed. \shir{Furthermore, it would be interesting to investigate the non-cooperative setting as in \citet{boursier2020selfish}, where we believe the design ideas in this work can still be of use, especially ADC.}

\section{Related Works}\label{sec:related}
\textbf{Decentralized MP-MAB.} Since \citet{liu2010distributed}, most MP-MAB works consider the homogeneous variant with player-independent arm utilities \citep{avner2014concurrent,rosenski2016multi,besson2018multi}. With implicit communications, \citet{boursier2019sic,proutiere2019optimal} prove regrets that approach the centralized ones. The homogeneous variant is fairly well understood by now. The heterogeneous MP-MAB problems \citep{kalathil2014decentralized,nayyar2016regret} with player-dependent arm utilities, on the other hand, remain largely open. The recent attempts have been summarized in Table~\ref{tbl:regret}, whose regrets are far from the (natural) centralized lower bound. Note that a similar idea of adaptive quantization is applied by \citet{boursier2019practical}, but the differential communication part in ADC is entirely novel and more critical to the overall performance.

All the aforementioned works are confined to the linear reward model. To the best of our knowledge, this work is the first to study general reward functions. Fairness is considered in \citet{bistritz2020my}  but with major differences elaborated in Appendix~\ref{sub_app:comp_maxmin}. Other MP-MAB variants, including ``stable'' allocations \citep{avner2016multi,darak2019multi}, no-sensing \citep{lugosi2018multiplayer, Shi2020aistats, bubeck2020coordination,bubeck2020cooperative}, and adversarial \citep{alatur2020multi,bubeck2020non,shi2020no}, fall out of our scope.

\textbf{Combinatorial MAB. } Since first presented by \citet{chen2013combinatorial}, many variants of stochastic CMAB have been investigated \citep{kveton2014matroid, kveton2015cascading}. Some recent works have also introduced Thompson Sampling into CMAB \citep{wang2018thompson,perrault2020statistical}. The study of lower bounds in CMAB has been active, e.g., for the linear reward function with correlated arms \citep{kveton2015combinatorial,degenne2016combinatorial} and independent arms \citep{combes2015combinatorial}. Recent attempts on lower bounds for general reward functions are reported by \citet{merlis2020tight}.

\shir{As illustrated in the design of BEACON, the decentralized MP-MAB model is closely related to CMAB, while these connections are largely ignored in the previous works. With more details presented in Appendix~\ref{app:similarity}, we here briefly note that in some sense, MP-MAB can be thought of as a decentralized version of CMAB, and this decentralized nature leads to additional challenges with collision-avoidance and information sharing. }

\section{Conclusion}\label{sec:cls}
In this work, we first investigated decentralized heterogeneous MP-MAB problems with linear reward function and proposed the BEACON algorithm. A novel adaptive differential (implicit) communication approach was designed and a batched structure was carefully crafted to incorporate the exploration principles from CUCB. With these novel ideas, BEACON achieved regrets that not only improve all prior regret bounds but in fact approach the centralized lower bound for the first time in the study of decentralized heterogeneous MP-MAB. Then, we extended the study to general reward functions and showed that BEACON can still obtain a regret of $O(\log(T))$ with simple modifications. Experimental results demonstrated that the gain of BEACON does not exist just in the theoretical analysis -- significant gains over state-of-the-art decentralized algorithms and achieving a comparable performance with the centralized benchmark have been empirically established. 

BEACON has demonstrated the intimate connection between MP-MAB and CMAB, two largely disjoint sub-fields of the MAB research. It is our hope that this work sparks future interest in investigating this fundamental connection and improving existing algorithms in both areas.

\bibliography{comb_bandit,bandit}
\bibliographystyle{apalike}

\newpage
\appendix

\allowdisplaybreaks

\section{MP-MAB and CMAB}\label{app:similarity}
The formulation of the MP-MAB model in the main paper shares several similarities with the CMAB model \citep{chen2013combinatorial,chen2016combinatorial,kveton2015tight}. However, these connections are largely ignored and unexplored in the previous literature, and we elaborate their similarities and differences here. First, the $K$ arms with different utilities for $M$ players can be equivalently interpreted as $MK$ base arms in the CMAB model. The matching set $\Sc$ can be viewed as one special set of super arms in CMAB, where each super arm is of size $M$ and must contain one arm from each player's $K$ arms. Furthermore, the semi-bandit feedback in CMAB assumes that observations from pulled arms are observable instead of the entire reward function, which is similar to the collision-sensing feedback discussed in the main paper. At last, the definition of reward function and regret also fit in the CMAB framework.

The key differences between MP-MAB and CMAB are in the structure of decentralized players. In CMAB, there is one centralized agent who decides all the actions and gets all the observations. However, MP-MAB is a decentralized setup where each player makes her own decisions and gets her own observations. From the perspective of decision making, the centralized configuration is more efficient as it will naturally choose the collision-free matchings. On the other hand, collision-avoidance is much harder in MP-MAB due to the decentralized decision making. To be more specific about the difference regarding the feedback, at time $t$, the centralized agent in CMAB makes decision based on the entire history $H(t) = \left\{s_m(\tau),O_{s_m(\tau),m}(\tau)\right\}_{m\in [M],1\leq \tau\leq t-1}$, while player $m$ in MP-MAB makes decision with her individual history $H_{m}(t) = \left\{s_m(\tau),O_{s_m(\tau),m}(\tau), \eta_{s_m(\tau)}(S(\tau))\right\}_{1\leq \tau\leq t-1}$. Obviously, information contained in $H_m(t)$ is more limited than that in $H(t)$. Note that $ \left\{\eta_{s_m(\tau)}(S(\tau))\right\}_{1\leq \tau\leq t-1}$ is omitted in $H(t)$ since it can be directly inferred by the centralized agent. Thus, MP-MAB can be viewed as a decentralized version of CMAB to some extent.

\section{Algorithmic Details of BEACON}\label{app:comm}
Some omitted algorithmic details of BEACON are presented in this section. 

\subsection{Orthogonalization Procedure}\label{sub_app:orthogonal}
In the orthogonalization (sometimes also referred to as the initialization) procedure, players estimate the number of players in the MP-MAB game and obtain distinct indices in a fully distributed manner. The initialization technique from \cite{proutiere2019optimal} is adopted in BEACON. It consists of two sub-phases: orthogonalization and rank assignment. The orthogonalization sub-phase aims at assigning each player with a unique external rank $k\in[K]$. It contains a sequence of blocks with length $K+1$, where each player attempts to fixate on arms without collision at first time step and states of fixation (successful or not) are broadcast (enabled by implicit communication). Note that in the original scheme \citep{proutiere2019optimal}, the broadcast is performed on the reserved arm $K$, which results in the need of $K>M$. To accommodate the scenarios with $K=M$, the broadcast can take place sequentially on arm $1$ to arm $K$. In the rank assignment sub-phase, a modified Round-Robin sequential hopping scheme helps the players convert their external ranks to internal ranks $m\in[M]$ and estimate the overall number of players $M$. Detailed algorithms can be found in \cite{proutiere2019optimal}. Using the same proofs in Lemma 1 and Lemma 2 in \cite{proutiere2019optimal}, we have the following performance characterization.
\begin{lemma}\label{lem:regret_init}
	The expected duration of the orthogonalization procedure in BEACON is less than $\frac{K^2M}{K-M}+2K$ time steps. Once the procedure completes, all players correctly learn the number of players $M$ and each of them is assigned with a unique index between $1$ and $M$.
\end{lemma}

\subsection{Detailed Communication Protocols}\label{subsec:supp_comm}
In this section, more details of the communication design are presented. First, as illustrated in Section~\ref{subsec:alg_comm}, the implicit communications are performed by having the ``receive'' player sample one arm and the ``send'' player either pull (create collision; bit $1$) or not pull (create no collision; bit $0$) the same arm to transmit one-bit information. Other players that are not communicating would fixate on other arms to avoid interruptions.  The arm(s) that the players pull for receiving or avoiding are referred to as ``communication arm(s)'', which is an arm-player matching and is assigned before the communication happens. In BEACON, the matching of communication arms for epoch $r>1$ is chosen as the exploration matching in the previous epoch, i.e., $S_{r-1}$. The benefit of this choice is that with the increasing explorations, $S_{r-1}$ would gradually become near-optimal with a high probability, which also leads to smaller communication losses.  Specifically,  in epoch $r$, follower $m>1$ (resp. the leader) communicates to the leader (resp. follower $m>1$) by either pulling or not pulling arm $s^{r-1}_1$ (resp. arm $s^{r-1}_m$), while the leader (resp. the follower $m$) stays on arm $s^{r-1}_1$ (resp. arm $s^{r-1}_m$) during receiving. To make this happen, in addition to the knowledge of index $s^{r-1}_m$ which is assigned to follower $m$ for explorations, index $s^{r-1}_1$ should also be communicated to the followers in the communication phase of epoch $r-1$.

Then, as illustrated in Section~\ref{subsec:alg_comm}, there are three kinds of information to be communicated, which are separately discussed in the following.

\textbf{Arm statistics.} The main idea of the adaptive differential communication (ADC) design is illustrated in Section~\ref{subsec:alg_comm}. However, two important ingredients are missing. The first is when follower $m$ quantizes the arm statistics $\tilde{\mu}^r_{k,m}$ from the collected sample mean $\hat{\mu}^r_{k,m}$ using $\lceil 1+ p^r_{k,m}/2\rceil$ bits. The least significant bit (LSB) is always ceiled to $1$ if $\lceil 1+ p^r_{k,m}/2\rceil$ bits cannot fully represent $\hat{\mu}^r_{k,m}$. We refer such process of quantizing $\tilde{\mu}^r_{k,m}$ as $\texttt{ceil}( \hat{\mu}^r_{k,m} )$ with $\lceil 1+ p^r_{k,m}/2\rceil$ bits. This process is needed for the later theoretical analysis to have $\tilde{\mu}^r_{k,m}\geq \hat{\mu}^r_{k,m}$.

The second missing component in ADC is referred to as the \textbf{signal-then-communicate} approach. The purpose of this approach is to synchronize the communication order and communication duration among players. It consists of two parts: the leader would first create a collision on the follower's communication arm to indicate the beginning of her statistics sharing; then, since the length of non-zero LSB at the end of $\delta^r_{k,m}$ is not fixed, after receiving the start signal, the follower $m$ would take the following approach to transmit $L$ bits ($L$ is however unknown to the leader), in which creating no collision indicates there are more bits to transmit while creating collision means the end of transmission:
\begin{align*}
    &\text{collision: start signal} \to \text{no collision} \to \text{one information bit} \to \cdots \\
    &\to \text{no collision} \to \text{one information bit} \to \text{collision: end signal}.
\end{align*}
Using no collision as an indicator also reduces the practical communication loss, as it avoids creating collisions during communications. In summary, with this signal-to-communicate approach, the original $L$-bits information of arm statistics would require no more than $(2L+2)$-bits.

\textbf{The chosen matching and leader's communication arm. } In epoch $r$, the leader needs to notify follower $m$ of both $s^r_m$ (for exploration) and $s^r_1$ (for communication in the next epoch). Similar to sharing arm statistics, the leader has to initiate the communication with a specific follower by creating a collision. Since both arm indices can be communicated via a fixed length of $\lceil\log_2(K)\rceil$ bits, they can be directly transmitted without using no-collisions to synchronize. Thus, with $K$ arms for each player, this part of communication can be done in $2\lceil\log_2(K)\rceil+1$ bits for each follower.

\textbf{Batch size. } A naive idea to transmit the batch size $p_r$ is to directly notify the followers of this number. However, the value of $p_r$ is at most $O(\log(T))$, which requires $O(\log\log(T))$ bits. With at most $O(\log(T))$ epochs of communication, directly sharing $p_r$ may lead to a dominating regret. Luckily, sharing $p_r$ only serves to let players explore the same length, which can be achieved by a much simpler and more efficient \textbf{stop-upon-signal} approach. Specifically, while $p_r$ is calculated by the leader, rather than broadcasting it to the followers via implicit collisions, she counts the exploration length herself and creates a collision on the exploration arm of each follower upon the end of exploration in this epoch. Upon perceiving collisions, followers become aware that the current exploration phase has ended.

\subsection{Algorithm for Followers}\label{sub_app:follower}
The detailed algorithm for the follower $m$ is presented in Algorithm.~\ref{alg:follower}.

\begin{algorithm}[htb]
	\caption{BEACON: Follower $m$}
	\label{alg:follower}
	\begin{algorithmic}[1]
		\State Set epoch counter $r\gets 0$; arm counter $[p^r_{k,m}]_{k\in[K]}\gets 0$; sample time $[T^r_{k,m}]_{k\in[M]}\gets 0$; communicated statistics $[\tilde{\mu}^r_{k,m}]_{k\in[M]}\gets 0$
		\State In order $k\in [K]$, play arm $[(m-1+k) \text{ mod } K]$ once  and update sample time $T^{r+1}_{k,m}\gets T^r_{k,m}+1$
		\While{not reaching the time horizon $T$}
		\State $r\gets r+1$
		\State $\forall k\in[K], p^r_{k,m}\gets \left\lfloor\log_2(T^r_{k,m})\right\rfloor$
		\State Update $\hat{\mu}^{r}_{k,m}$ as the sample mean from the first $2^{p^r_{k,m}}$ exploratory samples from arm $k$
		\Statex $\triangleright$ \textit{Communication Phase}
		\For{$k\in [K]$} 
		\If{$p^r_{k,m}>p^{r-1}_{k,m}$}
		\State {$\tilde{\mu}^r_{k,m} \gets \texttt{ceil}(\hat{\mu}^r_{k,m})$ with $\lceil 1+ p^r_{k,m}/2\rceil$} bits
		\State $\tilde{\delta}^r_{k,m}\gets \tilde{\mu}_{k,m}^r - \tilde{\mu}_{k,m}^{r-1}$
		\State $\texttt{Send}(\tilde{\delta}^r_{k,m}, 1)$
		\Else
		\State $\tilde{\mu}_{k,m}^r\gets \tilde{\mu}_{k,m}^{r-1}$
		\EndIf
		\EndFor
		\State $s^r_m \gets \texttt{Receive}(s^r_m,1)$
		\Statex $\triangleright$ \textit{Exploration Phase}
		\State Play arm $s^r_m$ until signaled
		\State Update $T^{r+1}_{s^r_m,m}\gets T^r_{s^r_m,m}+2^{p_r}$
		\EndWhile
	\end{algorithmic}
\end{algorithm}

\subsection{Sending and Receiving Protocols}\label{sub_app:send_receive}
The $\texttt{Send()}$ and $\texttt{Receive()}$ functions in Algorithms~\ref{alg:leader} and \ref{alg:follower} denote the protocols of sending and receiving information via forced collisions. In order to make this work self-contain, these two functions are illustrated in Algorithms~\ref{alg:send} and \ref{alg:receive}, while a more detailed illustration of the implicit communication approach can be found in \citet{boursier2019sic}. We further note that to better expose the sending and receiving structure, Algorithms~\ref{alg:send} and \ref{alg:receive} contain the key ideas in implicit communications, but omit some detailed protocols, e.g., the signal-then-communicate approach.

	\begin{algorithm}[htb]
	    \caption{$\texttt{Send()}$ for Player $m$}
	    \label{alg:send}
	    \begin{algorithmic}[1]
		\Require bit string $\boldsymbol{u} = [u_1, u_2, ..., u_{l}]$ with length $l$, receiver index $n$
		\State Initialization: player $m$'s communication arm $c_m$, player $n$'s communication arm $c_n$
		\For{$i = 1, 2, \cdots , l$}
		\If{$u_i=1$}
		\State Pull arm $c_n$ \Comment{collision for bit $1$}
		\Else
		\State Pull arm $c_m$ \Comment{no collision for bit $0$}
		\EndIf
		\EndFor
		\end{algorithmic}
	\end{algorithm}
	\begin{algorithm}[htb]
	    \caption{$\texttt{Receive()}$ for Player $n$}
	    \label{alg:receive}
	    \begin{algorithmic}[1]
		\Require bit string $\boldsymbol{u}'$ with length $l$, sender index $m$
		\State Initialization: player $n$'s communication arm $c_n$
		\For{$i = 1, 2, \cdots , l$}
		\State Pull arm $c_n$
		\If{ collision}
		\State $u'_i \gets 1$ \Comment{collision for bit $1$}
		\Else
		\State $u'_i \gets 0$ \Comment{no collision for bit $0$}
		\EndIf
		\EndFor
		\Ensure $\boldsymbol{u}'$
		\end{algorithmic}
	\end{algorithm}

\section{Reward Functions}
\subsection{Additional Examples}\label{sub_app:example_general}
Other than the proportional fairness function and minimal reward function gliven in the main paper, the following general (nonlinear) reward functions are also commonly adopted in real-world applications:
\begin{itemize}[leftmargin=*]\itemsep=0pt
    \item \textbf{Threshold}: $V(S,t) = \sum_{m\in[M]}\mathds{1}\left\{O_{s_m,m}(t)\geq \varphi_m\right\}$, where $\varphi_m$ is a player-dependent threshold. It characterizes the need of reaching certain thresholds, e.g., quality-of-service requirements, in cognitive radio systems;
	\item \textbf{Video quality-rate model}: $V(S,t) = \sum_{m\in[M]}U_m(O_{s_m,m}(t))$, where $U_m(O_{s_m,m}(t))$ is a piece-wise linear concave function on $[0,1]$ with decreasing slopes. It is typically used to describe video quality, and illustrates the decreasing of marginal utility with increased allocated resources;
	\item \textbf{Top-$L$ utility}: $V(S,t) = \max\left\{\sum_{m\in \mathcal{L}}O_{s_m(t),m}(t)|\mathcal{L} = [m_1,...,m_L]\subseteq [M], |\mathcal{L}|=L\right\}$, which features the highest sum of observations from any $L$ players.
\end{itemize}

\subsection{Comparison with Max-Min Fairness in \cite{bistritz2020my}}
\label{sub_app:comp_maxmin}
In \cite{bistritz2020my}, fairness is considered among the players in MP-MAB with a specific ``Max-Min'' fairness measure, which shares some similarities with the minimal reward function considered in this work but with major differences discussed in the following. 

\textbf{Reward function of \cite{bistritz2020my}.} The instantaneous system reward gained by the players of playing matching $S$ at time $t$ in \cite{bistritz2020my} is defined as
\begin{align*}
    V'(S,t) = \min_{m\in[M]}\left\{\Eb\left[O_{s_m,m}(t)\right]\right\} = \min\left\{\boldsymbol{\mu}_{S}\odot \boldsymbol{\eta}_S\right\},
\end{align*}
where expectations have already been taken {\em inside} the minimal function. To be consistent with the notation of this paper, the corresponding expected system reward of \cite{bistritz2020my} can be written as
\begin{align}\label{eqn:maxmin}
    V'_{\boldsymbol{\mu},S} = \Eb\left[V'(S,t) \right] =  \min_{m\in[M]}\left\{\boldsymbol{\mu}_{S}\odot \boldsymbol{\eta}_S\right\} = \min\left\{\boldsymbol{\mu}_{S}\odot \boldsymbol{\eta}_S\right\} = V'(S,t),
\end{align}
which does not differ from the instantaneous reward and remains the same with different utility distributions. 

\textbf{Reward function of this paper.} However, for the minimal reward function defined in this work, the instantaneous reward is 
\begin{align*}
    V(S,t) = \min_{m\in[M]}\left\{O_{s_m,m}(t)\right\},
\end{align*}
which is determined entirely by the instantaneously realized observations of players and does not incorporate any form of expectation. Further, the expected system reward is
\begin{align*}
    V_{\boldsymbol{\mu},S} = \Eb\left[V(S,t)\right] = \Eb\left[\min_{m\in[M]}\left\{O_{s_m,m}(t)\right\}\right],
\end{align*}
which does not have a uniform expression for different utility distributions. 

\textbf{Illustration of the differences.} The differences can be illustrated more clearly by assuming that the utility distributions are mutually independent Bernoulli distributions, i.e., $\phi_{k,m} = \text{Bernoulli}(\mu_{k,m})$, where $\mu_{k,m}\leq 1$ here is the probability that utility $1$ is generated by arm $(k,m)$. Then, the expected system reward function of \cite{bistritz2020my} and this work are shown in the following, respectively:
\begin{align*}
    \text{Max-Min fairness in \cite{bistritz2020my}: }&V'_{\boldsymbol{\mu},S} = \min_{m\in[M]}\{\mu_{s_m,m}\};\\
    \text{Minimal reward function in this work: }&V_{\boldsymbol{\mu},S} = \prod_{m\in[M]}\mu_{s_m,m}.
\end{align*}

Although the Max-Min fairness measure has several distinctions with the minimal reward function, its expected system reward function in Eqn.~\eqref{eqn:maxmin} also satisfies Assumptions~\ref{asp:reward}--\ref{asp:bound}. Thus, if we directly take Eqn.~\eqref{eqn:maxmin} as the expected sysmtem reward function (without explicitly defining the instantaneous reward function),  both the design and analysis of BEACON are applicable to the Max-Min fairness setting in \cite{bistritz2020my}. In this sense,  \cite{bistritz2020my} studied a special case of the general framework proposed in this work. Furthermore, since Theorem~\ref{thm:general} holds for this special case, this work improves the $O(\log\log(T)\log(T))$ regret provided by \cite{bistritz2020my} into a strictly $O(\log(T))$ regret.

\section{Experiment Details and Additional Results}\label{app:exp}
\subsection{Codes and Computational Resources}
The codes for the experiments are publicly available at \url{https://github.com/ShenGroup/MPMAB_BEACON}, along with detailed instructions. The experiments do not require heavy computations and all the simulations were performed by a common PC, which only took a few hours to complete in total.

\subsection{Detailed Experiment Settings}
All experimental results are averaged over $100$ independent runs and the utility distributions are taken as mutually independent Bernoulli distributions, i.e.,  $\phi_{k,m} = \text{Bernoulli}(\mu_{k,m})$.. The $5$-arms-$5$-players game adopted for the evaluation of the linear reward function shown in Fig.~\ref{fig:linear} is specified in the following, which is the same as the one adopted in \citet{boursier2019practical}:
\begin{equation*}
\boldsymbol{\mu}^T = [\mu_{k,m}]^T_{(k,m)}=
    \begin{bmatrix}
    0.5 & 0.49 & 0.39 & 0.29 & 0.5 \\
      0.5 & 0.49 & 0.39 & 0.29 & 0.19\\
      0.29 & 0.19& 0.5& 0.499& 0.39\\
      0.29& 0.49& 0.5& 0.5& 0.39\\
      0.49& 0.49& 0.49& 0.49& 0.5
    \end{bmatrix}.
\end{equation*}
The $8$-arms-$6$-players instance used in the simulation with the proportional fairness function and the minimal function in Figs.~\ref{fig:propotional} and \ref{fig:shortboard} is shown in the following:
\begin{equation*}
\boldsymbol{\mu}^T = [\mu_{k,m}]^T_{(k,m)}=\begin{bmatrix}
      0.45&0.49&0.59&0.17&0.37&0.86&0.94&0.98\\
      0.39&0.25&0.4&0.6&0.24&0.54&0.43&0.67\\
      0.39&0.33&0.8&0.01&0.12&0.2&0.61&0.77\\
      0.95&0.22&0.24&0.88&0.2&0.12&0.29&0.3\\
      0.69&0.89&0.25&0.59&0.43&0.18&0.01&0.84\\
      0.97&0.15&0.89&0.16&0.09&0.57&0.61&0.19
    \end{bmatrix}.
\end{equation*}

\subsection{METC Enhancements}\label{sub_app:metc_enhance}
To have a more fair comparison with METC \citep{boursier2019practical}, several enhancements and adjustments are conducted. First, all empirical enhancements introduced in the supplementary material of \cite{boursier2019practical} are implemented to achieve the best performance. Second, since METC is originally designed only for the linear reward function, enhancements are made to accommodate the adoption of general nonlinear reward functions. Specifically, for each active arm $(k,m)$, METC selects the empirically best matching $B_{k,m}$ containing arm $(k,m)$ w.r.t. the upper confidence bounds $\bar{\boldsymbol{\mu}}' = [\bar{\mu}'_{k,m}]_{(k,m)\in [K]\times [M]}$. The construction of $\bar{\boldsymbol{\mu}}'$ strictly follows the design from \cite{boursier2019practical}.  In its original form, this step is confined to the linear reward function as 
\begin{align*}
    B_{k,m} \gets \argmax_{S\in \Sc, s_m=k}\left\{\sum\nolimits_{n\in[M]}\bar{\mu}'_{s_n,n}\right\}.
\end{align*}
We apply the same principle to the general reward functions by assuming an enhanced oracle such that
\begin{align*}
    B_{k,m} \gets \texttt{OracleEnhanced}(\boldsymbol{\bar{\mu}}',k,m)  \gets \argmax_{S\in \Sc, s_m=k}\left\{v(\bar{\boldsymbol{\mu}}'_S\odot \boldsymbol{\eta}_S)\right\}.
\end{align*}
The same idea is applied to the procedure of eliminating arms in METC. Note that the requirement for this oracle is much higher than the one used in BEACON, since it needs to output a specific exploration matching for each active arm, instead of only one matching as in BEACON.

\subsection{Additional Experimental Results}
\begin{figure*}[htb]
	\centering
	\subfigure[Linear reward function.]{ \includegraphics[width=0.33\linewidth]{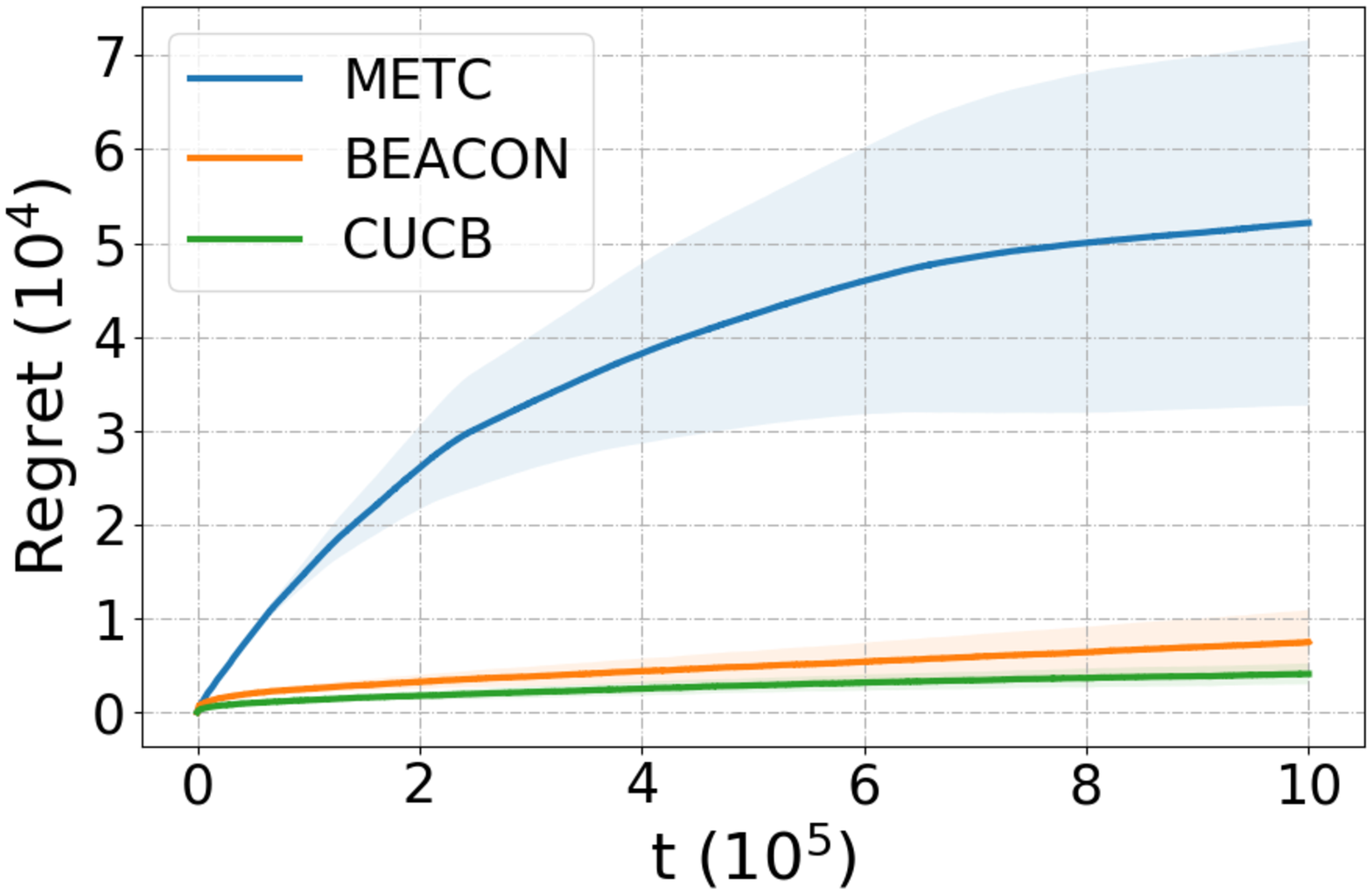}\label{fig:linear_comp}}
	\subfigure[Minimal function.]{ \includegraphics[width=0.33\linewidth]{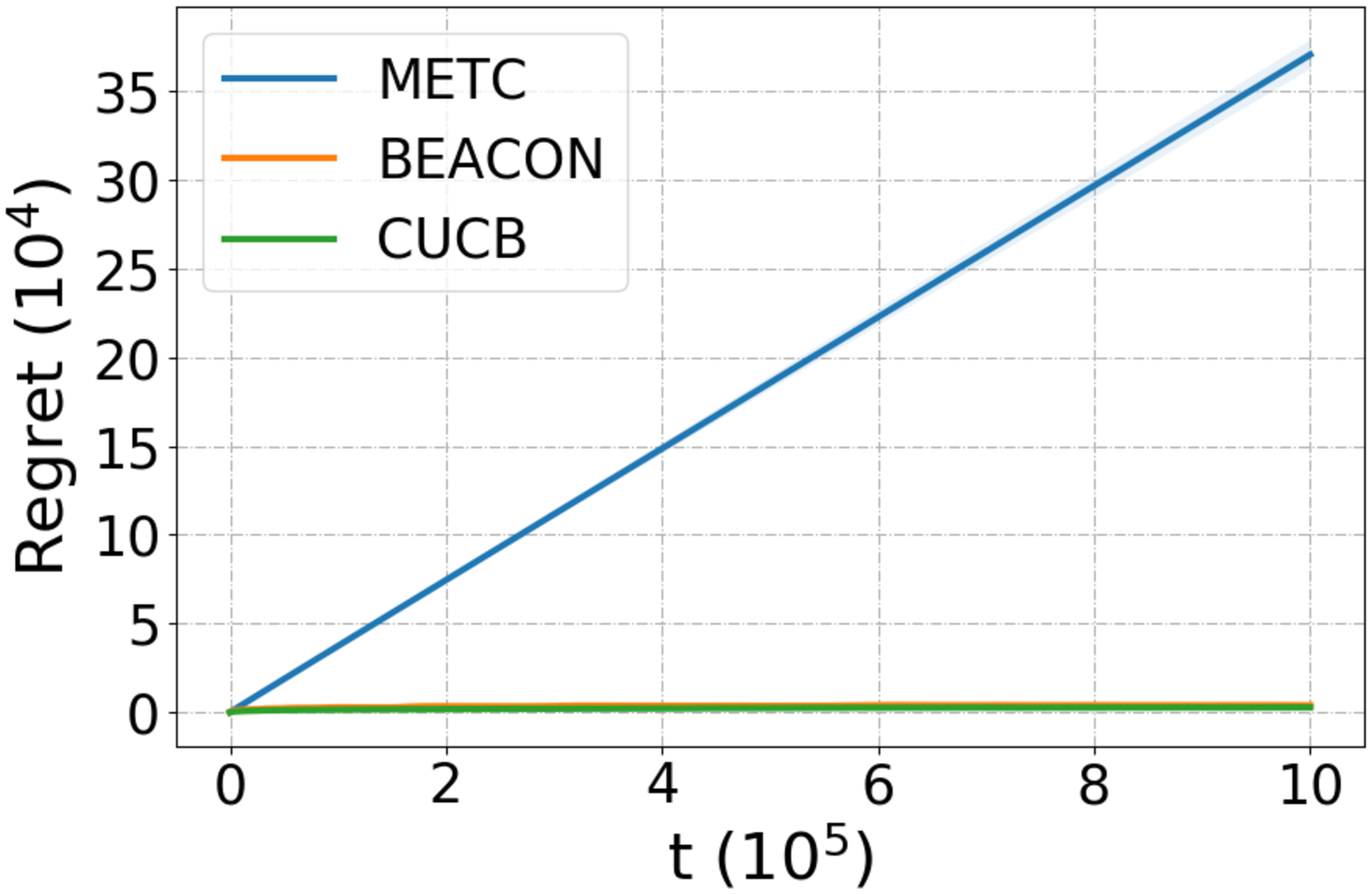}\label{fig:shortboard_comp}}
	\caption{Complete regret comparisons of Figs.~\ref{fig:linear} and \ref{fig:shortboard}. The regret curves of CUCB and BEACON are sometimes too close to each other to be distinguished.}
\end{figure*}

First, Figs.~\ref{fig:linear_comp} and \ref{fig:shortboard_comp} are the complete versions of Figs.~\ref{fig:linear} and \ref{fig:shortboard}, where the significant advantage of BEACON over METC is illustrated more clearly. 

\begin{figure*}[htb]
	\centering
	\subfigure[Linear reward function.]{ \includegraphics[width=0.32\linewidth]{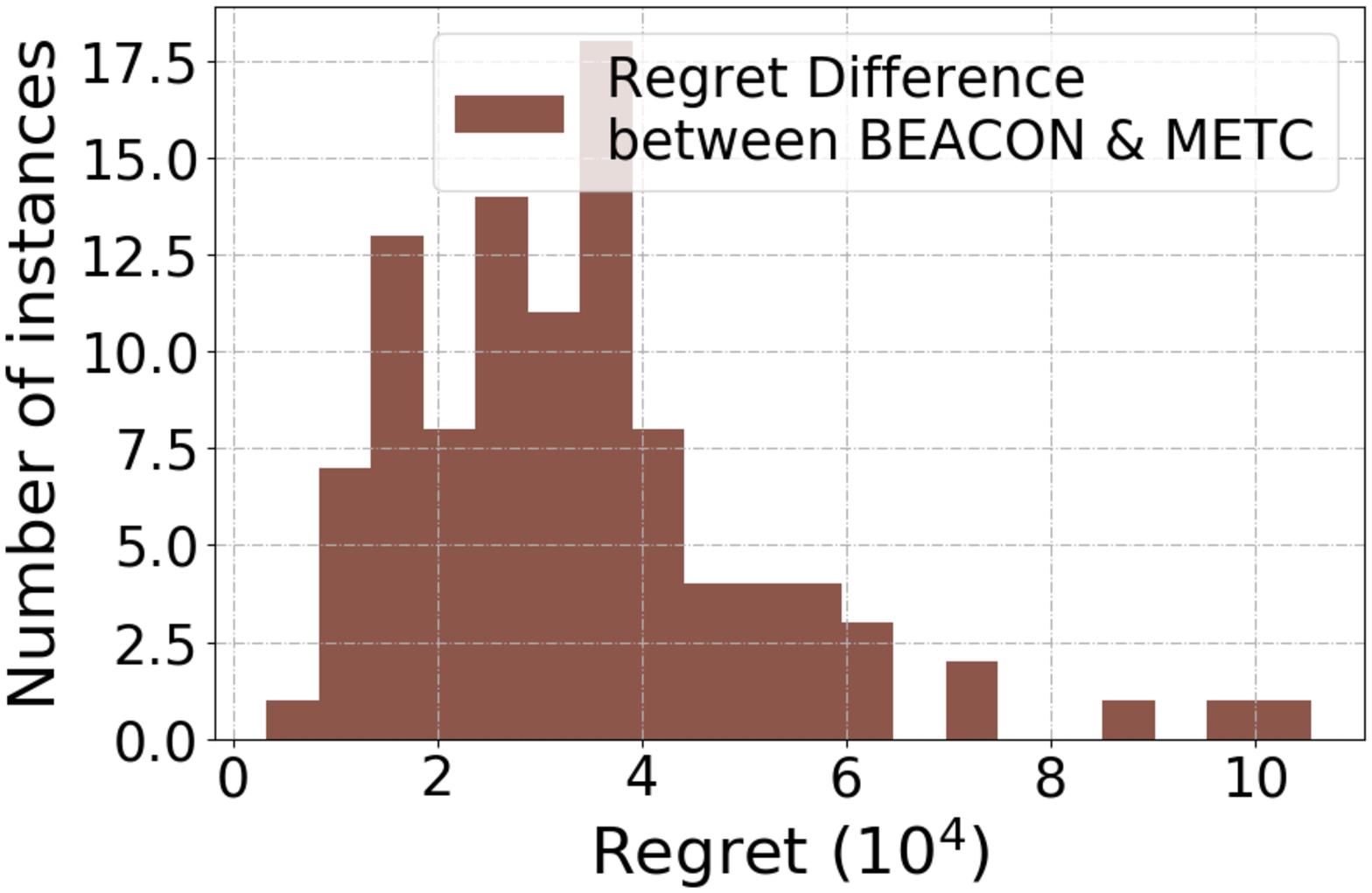}\label{fig:linear_diff}}
	\subfigure[Linear reward function and a large game.]{ \includegraphics[width=0.32\linewidth]{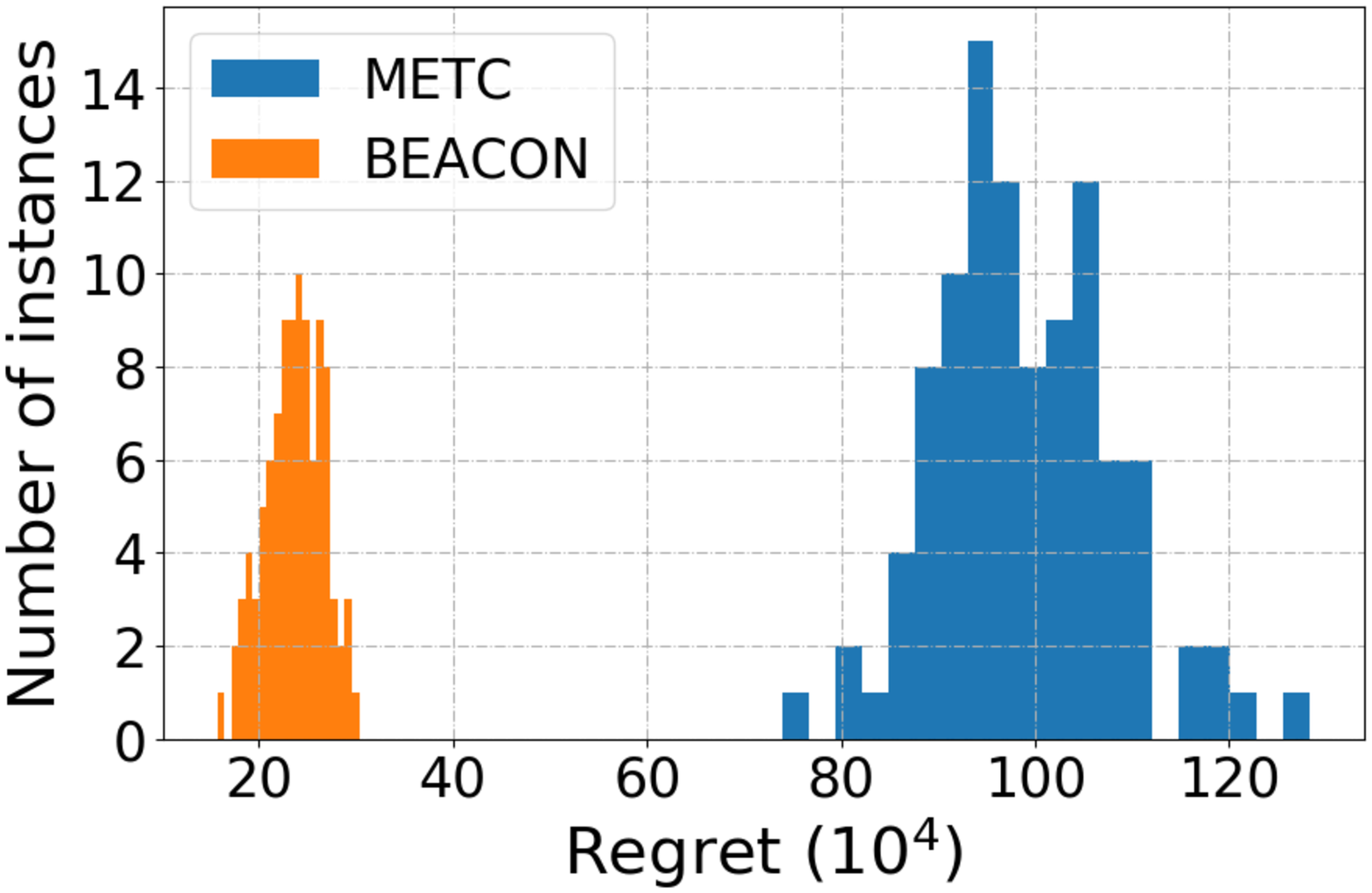}\label{fig:linear_hist_large}}
	\subfigure[Linear reward function and a large game.]{ \includegraphics[width=0.32\linewidth]{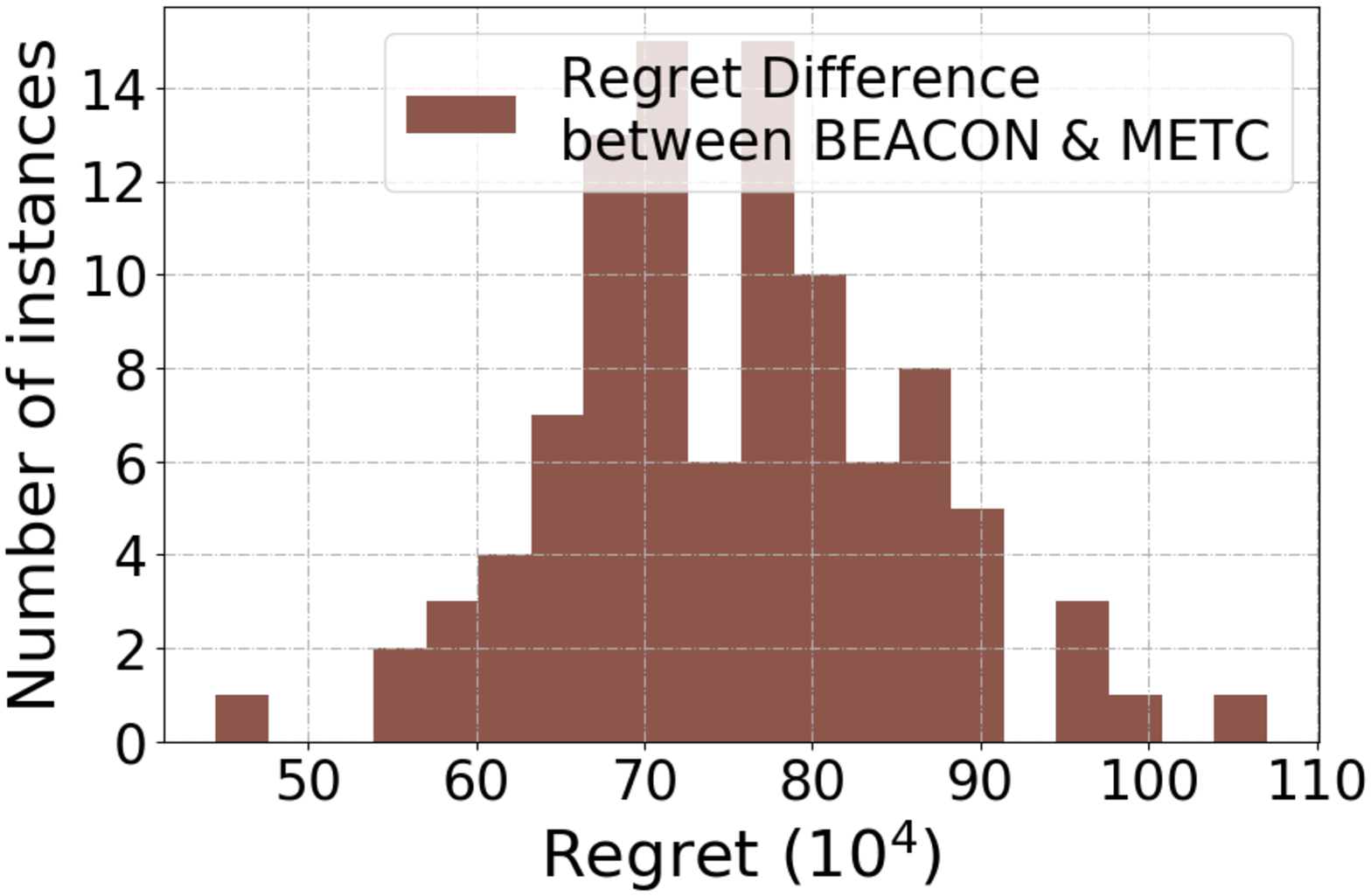}\label{fig:linear_diff_large}}
	\caption{Regret histograms with the linear reward function. (a) is the regret difference corresponding to Fig.~\ref{fig:linear_random}. (b) is the cumulative result from $100$ randomly generated instances with a large game setting with $M=10$ and $K = 30$, and (c) is the corresponding regret difference for (b).}
\end{figure*}

Then, Fig.~\ref{fig:linear_diff} presents the regret differences between BEACON and METC corresponding to Fig.~\ref{fig:linear_random}. A large game setting with $M=10, K=30, T=10^6$ are evaluated using $100$ randomly generated instances with results reported in Fig.~\ref{fig:linear_hist_large} and \ref{fig:linear_diff_large}. We can observe that the performance of BEACON is stable with this large game setting and is still significantly better than METC, which further demonstrates the advantages of BEACON.

\begin{figure*}[htb]
	\centering
	\subfigure[Proportional fairness function.]{ \includegraphics[width=0.32\linewidth]{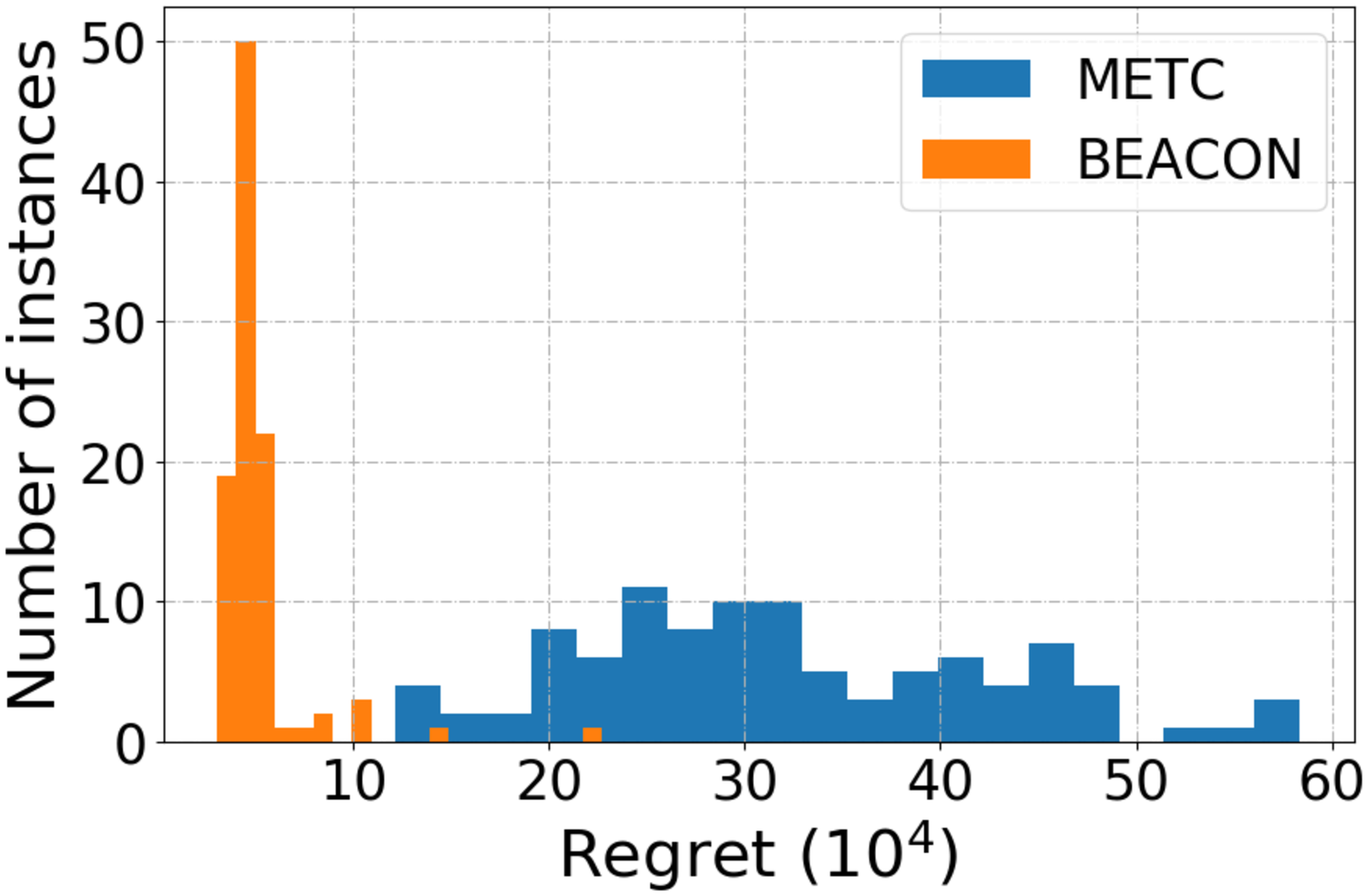}\label{fig:prop_hist}}
	\subfigure[Proportional fairness function.]{ \includegraphics[width=0.32\linewidth]{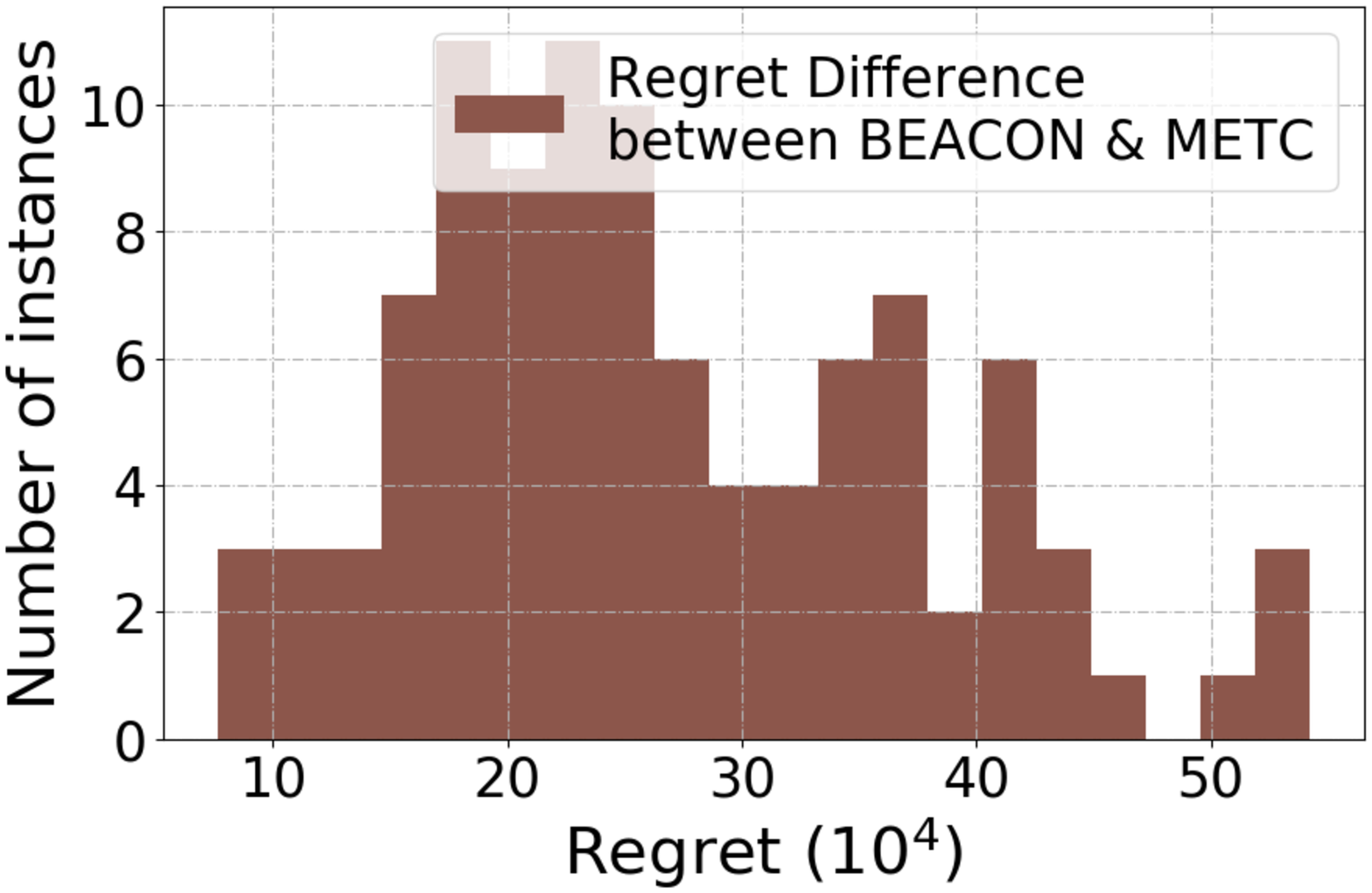}\label{fig:prop_diff}}
	\caption{Regret histograms with the proportional fairness function. (a) is the cumulative regret result from $100$ randomly generated instances, and (b) is the corresponding regret difference for (a).}
\end{figure*}

Also, $100$ randomly generated instances with $M=5,K = 6, T=10^6$ are used to evaluate the performance of BEACON and METC in dealing with the proportional fairness function. The histogram of the regrets is given in Fig.~\ref{fig:prop_hist} along with the histogram of the regret differences in Fig.~\ref{fig:prop_diff}, the latter of which gives a more definitive illustration of the advantage of BEACON. It can be observed that BEACON effectively deals with this proportional fairness function and outperforms METC uniformly across all realizations, which again proves the stable performance of BEACON in dealing with general reward functions.

In addition to the theoretical comparison of regret analyses given in Table \ref{tbl:regret}, we also provide some empirical explanations of BEACON's advantages over METC. First, the differential communication design is the key to lower communication losses. In fact, in the experiments, the statistical difference to be communicated, i.e., $\delta^r_{k,m}$, is much smaller than the theory dictates. We have frequently observed that there are only one to two non-zero bits to be communicated. Second, for explorations, METC adopts the strategy of arm elimination, while BEACON does not explicitly eliminate arms but instead uses confidence bounds to balance exploration and exploitation. From the experimental results, the arm elimination approach in METC is more costly than the exploration strategy in BEACON. This improvement again illustrates the importance of the connection between MP-MAB and CMAB.

\section{Proof for Theorem~\ref{thm:general}}
We begin with the analysis of BEACON with general reward functions, i.e., Theorem~\ref{thm:general}, since it is more intuitive than the one for the linear reward function, i.e., Theorem~\ref{thm:linear}. The latter follows the same spirit of the former but is carefully tailored to the linear reward function.

The complete version of Theorem~\ref{thm:general} is first presented in the following.
\begin{theorem}[\textbf{Complete version of Theorem~\ref{thm:general}}]\label{thm:general_full}
	Under Assumptions~\ref{asp:reward}, \ref{asp:mono}, and  \ref{asp:bound}, the regret of BEACON is upper bounded as
	\begin{align*}
		&R(T) \leq  \sum_{(k,m)\in [K]\times [M]}\left[\frac{28\Delta^{k,m}_{\min}\ln(T)}{(f^{-1}(\Delta^{k,m}_{\min}))^2}+ \int_{\Delta^{k,m}_{\min}}^{\Delta^{k,m}_{\max}} \frac{28\ln(T)}{(f^{-1}(x))^2}\mathrm{d}x+
		4KM\Delta^{k,m}_{\max}\right]\\
		&+\frac{6}{\ln 2}M^2K\log_2(K)\Delta_c\ln(T) + \frac{18}{\ln 2} MK \Delta_c\ln(T) +MK\Delta_c+\left(\frac{K^2M}{K-M}+2K\right)\Delta_{c}+K\Delta_{\max}\\
		& =\tilde{O}\left( \sum_{(k,m)\in[K]\times [M]}\left[\frac{\Delta^{k,m}_{\min}}{(f^{-1}(\Delta^{k,m}_{\min}))^2}+ \int_{\Delta^{k,m}_{\min}}^{\Delta^{k,m}_{\max}} \frac{1}{(f^{-1}(x))^2}\mathrm{d}x\right]\log(T)+ M^2K\Delta_c\log(T) \right)\\
		&=  \tilde{O}\left(\sum_{(k,m)\in[K]\times [M]}\frac{\Delta^{k,m}_{\max}\log(T)}{(f^{-1}(\Delta^{k,m}_{\min}))^2}+M^2K\Delta_c\log(T)\right).
	\end{align*}
\end{theorem}   
To facilitate the proof, we introduce (or recall) the following notations:
\begin{align*}
	&V_{\boldsymbol{\mu},*} =\max \{V_{\boldsymbol{\mu},S}|S\in \Sc\}=\max\{v(\boldsymbol{\mu}_S\odot\boldsymbol{\eta}_S) |S\in \Sc\}\text{: the optimal reward value};\\
	&\Sc_* = \{S|S\in \Sc, V_{\boldsymbol{\mu},S} = V_{\boldsymbol{\mu},*} \} \hfill\text{: the set of the optimal matchings};\\
	&\Sc_c = \{S|\exists m\neq n, s_m=s_n\} \text{: the set of matchings with collisions};\\
	&\Sc_b = \Sc\backslash (\Sc_*\cup \Sc_c) \text{: the set of collision-free suboptimal matchings};\\
	&\Delta^{k,m}_{\min} = V_{\boldsymbol{\mu},*} - \max\{V_{\boldsymbol{\mu},S}|S\in\Sc_b, s_m = k \};\\
	&\Delta^{k,m}_{\max} = V_{\boldsymbol{\mu},*} - \min\{V_{\boldsymbol{\mu},S}|S\in\Sc_b, s_m = k \};\\
	&\Delta_{\min} = \min \{\Delta^{k,m}_{\min}\}\text{: the smallest reward gap among collision-free matchings};\\
	&\Delta_{\max} = \max \{\Delta^{k,m}_{\max}\}\text{: the largest reward gap among collision-free matchings};\\
	&\Delta_{c} =  V_{\boldsymbol{\mu},*}-\min\{V_{\boldsymbol{\mu},S}|S\in\Sc_c\}\leq f(1)\text{: the largest possible per-step loss upon collisions}.
\end{align*}
\begin{proof}[Proof for Theorems~\ref{thm:general} and \ref{thm:general_full}]
	The overall regret $R(T)$ can be decomposed into three parts: the exploration regret $R_e(T)$, the communication regret $R_c(T)$, and the other regret $R_o(T)$, i.e.,
	\begin{align*}
		R(T) = R_e(T) + R_c(T) +R_o(T).
	\end{align*}
	The exploration regret $R_e(T)$ and the communication regret $R_c(T)$ are caused by exploration and communication phases, respectively, and are analyzed in the following subsections. The other regret $R_o(T)$ contains the regret caused by orthogonalization and activation, i.e., the explorations before epoch $1$, and can be easily bounded as
	\begin{align}
	\label{eqn:apdx1_Ro}
		R_o(T) \leq \left(\frac{K^2M}{K-M}+2K\right)\Delta_{c}+K\Delta_{\max},
	\end{align}
	where the first term is the regret from orthogonalization (Lemma~\ref{lem:regret_init}) and the second term is the regret from activation.
	
	With Lemmas~\ref{lem:comm_regret} and \ref{lem:expl_regret}, which bound $R_c(T)$ and $R_e(T)$ respectively, established in the following subsections, and the bound on $R_o(T)$ in Eqn.~\eqref{eqn:apdx1_Ro}, Theorems~\ref{thm:general} and \ref{thm:general_full} can be directly proved.
\end{proof}

\subsection{Communication Regret}
\begin{lemma}\label{lem:comm_regret}
	For BEACON, under time horizon $T$, the cumulative length of all communication phases $D_c$ is bounded as
	\begin{align*}
		\Eb[D_c]  \leq \frac{6}{\ln 2}M^2K\log_2(K)\ln(T) + \frac{18}{\ln 2} MK \ln(T) +MK,
	\end{align*}
	and the communication loss $R_c(T)$ is bounded as
	\begin{align*}
		R_c(T) \leq \Eb[D_c]\Delta_c  \leq \frac{6}{\ln 2}M^2K\log_2(K)\Delta_c\ln(T) + \frac{18}{\ln 2} MK \Delta_c\ln(T) +MK\Delta_c.
	\end{align*}
\end{lemma}
\begin{proof}[Proof for Lemma~\ref{lem:comm_regret}]
	As illustrated in Section~\ref{subsec:alg_comm} and Appendix~\ref{app:comm}, communication phases consist of three parts of information sharing: arm statistics $\tilde{\mu}^r_{k,m}$, the chosen matching $S_r$, and the batch size parameter $p_r$. With the detailed communication protocol described in Appendix~\ref{app:comm}, we bound the communication lengths of the aforementioned three parts,  respectively.
	
	\textbf{Part I: Arm statistics.} We take arm $(k,m), m\neq 1$ as an example. In epoch $1$, $\tilde{\mu}_{k,m}^0$ is initialized as $0$ while $\bar{\mu}_{k,m}^1$ is the value of one random utility sample from arm $(k,m)$. With $p^1_{k,m} = \lfloor\log_2(T_{k,m}^1)\rfloor = \lfloor\log_2(1)\rfloor =  0$, $\tilde{\mu}_{k,m}^1$ is quantized from  $\hat{\mu}_{k,m}^1$ with $1+p^1_{k,m} = 1$ bit. The difference $\tilde{\delta}^1_{k,m} = \tilde{\mu}^1_{k,m} - \tilde{\mu}^0_{k,m} = \tilde{\mu}^1_{k,m}$ is transmitted and it contains only $1$ bit.
	
	 In epoch $r> 1$, if $p^r_{k,m}> p^{r-1}_{k,m}$, i.e., $p^r_{k,m}= p^{r-1}_{k,m}+1$, arm statistics of arm $(k,m)$ should be communicated via the truncated version of the difference $\tilde{\delta}^r_{k,m} = \tilde{\mu}^r_{k,m} - \tilde{\mu}^{r-1}_{k,m}$. Then, we can bound the duration of communication through bounding $\tilde{\delta}^r_{k,m}$. Specifically, it holds that
	 \begin{align*}
	 	|\tilde{\delta}^r_{k,m}| &= |\tilde{\mu}^r_{k,m} - \tilde{\mu}^{r-1}_{k,m}|\\
	 	& = |\tilde{\mu}^r_{k,m} - \hat{\mu}^r_{k,m} - (\tilde{\mu}^{r-1}_{k,m}-\hat{\mu}^{r-1}_{k,m})+(\hat{\mu}^r_{k,m}-\hat{\mu}^{r-1}_{k,m})|\\
	 	& \leq |\tilde{\mu}^r_{k,m} - \hat{\mu}^r_{k,m}| + |\tilde{\mu}^{r-1}_{k,m}-\hat{\mu}^{r-1}_{k,m}|+|\hat{\mu}^r_{k,m}-\hat{\mu}^{r-1}_{k,m}|\\
	 	&\overset{(a)}{\leq} \sqrt{\frac{1}{2^{p^r_{k,m}}}}+\sqrt{ \frac{1}{2^{p^r_{k,m}-1}}}+|\hat{\mu}^r_{k,m}-\hat{\mu}^{r-1}_{k,m}|,
	 \end{align*}
 	where inequality (a) is due to the quantization process specified Section~\ref{subsec:supp_comm}, i.e., $\tilde{\mu}_{k,m}^r=\texttt{ceil}(\hat{\mu}_{k,m}^r)$ with $\lceil 1+ p^r_{k,m}/2\rceil$ bits. This quantization leads to a quantization error of at most ${2^{-p^r_{k,m}/2}}$. Further, denoting $\gamma^{k,m}_{\tau}$ as the $\tau$-th random utility sample from arm $(k,m)$ during exploration phases, we can rewrite the difference $\hat{\mu}^r_{k,m}-\hat{\mu}^{r-1}_{k,m}$ as
 	\begin{align*}
 		\hat{\mu}^r_{k,m}-\hat{\mu}^{r-1}_{k,m} &= \frac{\sum_{\tau = 1}^{2^{p^r_{k,m}}}\gamma^{k,m}_{\tau}}{2^{p^r_{k,m}}} - \frac{\sum_{\tau = 1}^{2^{p^r_{k,m}-1}}\gamma^{k,m}_{\tau}}{2^{p^r_{k,m}-1}}\\
 		& = \frac{\sum_{\tau = 1}^{2^{p^r_{k,m}-1}}\gamma^{k,m}_{\tau}+\sum_{\tau = 1+2^{p^r_{k,m}-1}}^{2^{p^r_{k,m}}}\gamma^{k,m}_{\tau}}{2^{p^r_{k,m}}} - \frac{\sum_{\tau = 1}^{2^{p^r_{k,m}-1}}\gamma^{k,m}_{\tau}}{2^{p^r_{k,m}-1}}\\
 		& = \frac{\sum_{\tau = 1+2^{p^r_{k,m}-1}}^{2^{p^r_{k,m}}}\gamma^{k,m}_{\tau}-\sum_{\tau = 1}^{2^{p^r_{k,m}-1}}\gamma^{k,m}_{\tau}}{2^{p^r_{k,m}}}\\
 		& =\frac{1}{2^{p^r_{k,m}}}\sum_{\tau = 1}^{2^{p^r_{k,m}-1}}\left(\gamma^{k,m}_{\tau+2^{p^r_{k,m}-1}}-\gamma^{k,m}_{\tau}\right)
 	\end{align*} 
 	which is a $\frac{1}{\sqrt{2^{p^r_{k,m}+1}}}$-sub-Gaussian random variable since the utility samples are independent across time. Thus, we can further derive that, with a dummy variable $x\geq \sqrt{\ln 2}$, 
 	\begin{align*}
 		&\Pb\left(\left|\hat{\mu}^r_{k,m}-\hat{\mu}^{r-1}_{k,m}\right|\geq \sqrt{\frac{x^2}{2^{p^r_{k,m}}}}\right)\leq 2\exp\left[-2^{p^r_{k,m}}\frac{x^2}{2^{p^r_{k,m}}}\right] \leq  2\exp[-x^2]\\
 		\Rightarrow & \Pb\left(|\tilde{\delta}^r_{k,m}|\geq \sqrt{\frac{1}{2^{p^r_{k,m}}}}+\sqrt{ \frac{1}{2^{p^r_{k,m}-1}}}+\sqrt{\frac{x^2}{2^{p^r_{k,m}}}}\right)\leq 2\exp[-x^2]\\
 		\overset{(a)}{\Rightarrow} & \Pb\left(L^r_{k,m} \geq 3+\frac{p^r_{k,m}}{2}+\log_2\left(\frac{1+\sqrt{2}+x}{\sqrt{2^{p^r_{k,m}}}}\right)\right)\leq 2\exp[-x^2]\\
 		\Rightarrow & \Pb\left(L^r_{k,m} \geq 3+\log_2\left(3+x\right)\right)\leq 2\exp[-x^2]\\
 		\Rightarrow & \Pb\left(L^r_{k,m} \leq 3+\log_2\left(3+x\right)\right)\geq 1-2\exp[-x^2]\\
 		\overset{(b)}{\Rightarrow} & \Pb\left(L^r_{k,m} \leq l\right)\geq 1-2\exp\left[-(2^{l-3}-3)^2\right]
 	\end{align*}
 	where $L^r_{k,m}$ in implication (a) is the length of the truncated version $|\tilde{\delta}^r_{k,m}|$ and is upper bounded by 
 	\begin{align*}
 		L^r_{k,m}&\leq \lceil 1+p^r_{k,m}/2\rceil-\lfloor\log_2(1/|\tilde{\delta}^r_{k,m}|)\rfloor\\
 		&\leq 3+p^r_{k,m}/2+\log_2(|\tilde{\delta}^r_{k,m}|).
 	\end{align*}
	In deriving (b), we substitute the variable $3+\log_2(3+x)$ with $l$, which satisfies that $l\geq 3+\log_2(3+\sqrt{\ln 2})$, and thus equivalently $x = 2^{l-3}-3$. With the above results and viewing $L^r_{k,m}$ as a random variable, we have that its cumulative distribution function (CDF) $F_{L^r_{k,m}}(l)$ satisfies the following property:
	\begin{align*}
		\forall l \geq 5>3+\log_2(3+\sqrt{\ln 2}), F_{L^r_{k,m}}(l) = \Pb\left(L^r_{k,m} \leq l\right)\geq 1-2\exp\left[-(2^{l-3}-3)^2\right].
	\end{align*}
	Using the property of CDF, we can bound the expectation of $L_{k,m}^r$ as
	\begin{align*}
		\Eb\left[L_{k,m}^r\right] &= \sum_{l = 0}^{\infty} (1-F_{L^r_{k,m}}(l))\\
		&\leq 6+ \sum_{l=6}^{\infty}2\exp\left[-(2^{l-3}-3)^2\right]\\
		&\leq 6+ \int_{l=5}^{\infty}2\exp\left[-(2^{l-3}-3)^2\right]\mathrm{d} l\\
		&\leq 7.
	\end{align*}

	Thus, we have that in expectation, the truncated version of $|\tilde{\delta}^r_{k,m}|$ has a length that is less than $7$ bits. In addition, $1$-bit information should also be transmitted to indicate the sign of $\tilde{\delta}^r_{k,m}$. As a summary, in expectation, $8$ bits is sufficient to represent the truncated version of $\tilde{\delta}^r_{k,m}$, 
	
	With overall time horizon of $T$, there are at most $\log_2(T)$ statistics updates of arm $(k,m)$ in addition to the first epoch. The expected communication duration for arm statistics $D_s$ is bounded as
	\begin{align}
		\Eb\left[D_s\right] &\overset{(a)}{=}  \underbrace{MK}_{\text{epoch $r=1$}}+  \underbrace{\Eb\left[\sum_{r}\sum_{(k,m):p^r_{k,m}>p^{r-1}_{k,m}}(2 + 2(L_{k,m}^r+1))\right]}_{\text{epoches $r>1$}}\notag\\ 
		&{\leq} MK+ (2+2\times 8) MK\log_2(T) \notag\\ 
		& \leq  18 MK\log_2(T)+MK \notag\\ 
		& = \frac{18}{\ln 2} MK \ln(T) +MK\label{eqn:d_s},
	\end{align}
	where equation (a) takes the signal-then-communicate protocol described in Appendix~\ref{app:comm} into consideration, where transmitting $\tilde{\delta}^r_{k,m}$ consists of $1$ step of the leader notifying the follower to start, $(L^r_{k,m}+1)$ steps of the truncated version of $\tilde{\delta}^r_{k,m}$ and correspondingly $(L^r_{k,m}+2)$ steps of synchronization between the leader and follower.
	
	\textbf{Part II \& III: Matching choice and batch size.} These two parts of communications are relatively easy to bound. In each epoch $r$, the leader initiates and then transmits two arm indices ($s^r_1$ and $s^r_m$) to each follower $m$, thus, the communication duration $D_m$ for matching assignments is bounded as
	\begin{align}
		D_m  &= \sum_{r}(M-1)(1+2\lceil\log_2(K)\rceil) \notag\\
		& \leq (M-1)(2\log_2(K)+3)MK\log_2(T) \notag\\
		& < \frac{1}{\ln 2}M^2K(2\log_2(K)+3)\ln(T).\label{eqn:d_m}
	\end{align}

	For the communication duration $D_b$ for the batch size, as illustrated in Appendix~\ref{app:comm}, the leader notifies followers to stop exploring by sending stopping signals. Thus, it holds that
	\begin{align}\label{eqn:d_b}
		D_b &= \sum_{r}(M-1) \leq (M-1)MK\log_2(T) < \frac{1}{\ln 2}M^2K\ln(T).
	\end{align}
	
	By combining Eqns.~\eqref{eqn:d_s}, \eqref{eqn:d_m} and \eqref{eqn:d_b}, Lemma~\ref{lem:comm_regret} can be obtained as
	\begin{align*}
		\Eb[D_c]  &= \Eb[D_s] +\Eb[D_m] +\Eb[D_b] \\
		&\leq  \frac{18}{\ln 2} MK \ln(T) +MK + \frac{1}{\ln 2}M^2(2\log_2(K)+3)K\ln(T)+\frac{1}{\ln 2}M^2K\ln(T)\\
		& \leq \frac{6}{\ln 2}M^2K\log_2(K)\ln(T) + \frac{18}{\ln 2} MK \ln(T) +MK.
	\end{align*}
\end{proof}

\subsection{Exploration Regret}
\begin{lemma}\label{lem:expl_regret}
	For BEACON, under time horizon $T$, the exploration regret is upper bounded as
	\begin{align*}
		R_e(T) &\leq \sum_{(k,m)\in [K]\times [M]}\left[\frac{28\Delta^{k,m}_{\min}\ln(T)}{(f^{-1}(\Delta^{k,m}_{\min}))^2}+ \int_{\Delta^{k,m}_{\min}}^{\Delta^{k,m}_{\max}} \frac{28\ln(T)}{(f^{-1}(x))^2}\mathrm{d}x+
		4 KM\Delta^{k,m}_{\max}\right].
	\end{align*}
\end{lemma}
\begin{proof}[Proof for Lemma~\ref{lem:expl_regret}]
	The following proof is inspired by the proof for CUCB in \cite{chen2013combinatorial}. However, \cite{chen2013combinatorial} does not consider the batched structure, which introduces additional challenges for the proof here. To better characterize the exploration regret, we introduce the following notations:
	\begin{align*}
		&\Sc^{k,m}_{b} = \{S|S\in \Sc_b, s_m=k\}= \{S^{k,m}_{1},..., S^{k,m}_{N(k,m)}\};\\
		&\Delta^{k,m}_{n} = V_{\boldsymbol{\mu},*} - V_{\boldsymbol{\mu}, S^{k,m}_{n}}, \forall n\in\{1,...,N(k,m)\},
	\end{align*}
	where $\Sc^{k,m}_{b}$ is the set of collision-free sub-optimal matchings that contain arm $(k,m)$ and we denote its size as $N(k,m)$. $\Delta^{k,m}_{n}$ denotes the sub-optimality gap of the matching $S^{k,m}_{n}$.
	In the following proof, we re-arrange the set $\Sc^{k,m}_{b} = \{S^{k,m}_{1},..., S^{k,m}_{N(k,m)}\}$ in a decreasing order w.r.t. the gap $\Delta^{k,m}_{n}$, i.e., if $n_1\geq n_2$, $\Delta^{k,m}_{n_1}\leq \Delta^{k,m}_{n_2}$. Also, for convenience, we denote $\Delta^{k,m}_{N(k,m)+1}=0$. Furthermore, it naturally holds that $\Delta^{k,m}_{\min} = \Delta^{k,m}_{N(k,m)}$ and $\Delta^{k,m}_{\max} = \Delta^{k,m}_{1}$.
	
	We denote $q_{n}^{k,m}, \forall n\in\{1,...,N(k,m)\}$ as the integer such that
	\begin{align*}
		2^{q^{k,m}_{n}-1} \leq  \frac{14\ln(T)}{(f^{-1}(\Delta^{k,m}_n))^2} < 2^{q^{k,m}_{n}} < \frac{28\ln(T)}{(f^{-1}(\Delta^{k,m}_n))^2}.
	\end{align*}
	In addition, we define $q_{0}^{k,m} = 0$ and $q_{N(k,m)+1}^{k,m} = \lceil\log_2(T)\rceil$. Note that with the above definition of $q^{k,m}_n$, it holds that 
	\begin{align}\label{eqn:qkm}
		\forall p\geq q^{k,m}_n, f\left(2\sqrt{\frac{3\ln t_r}{2^{p+1}}}+\sqrt{\frac{1}{2^{p}}}\right)&\leq f\left(3\sqrt{\frac{3\ln t_r}{2^{p+1}}}\right)\leq f\left(3\sqrt{\frac{3\ln T}{2^{p+1}}}\right)< \Delta^{k,m}_n,
	\end{align}
	which is a key property that is utilized in the subsequent proofs.
	
	For epoch $r$, we define the ``representative arm'' $\rho_r = (s^r_m,m)$ as one of the arms in $S_r$ such that $p^r_{s^r_m,m} = p_r$. If there are more than one arm in $S_r$ with arm counter $p_r$, $\rho_r$ is randomly chosen from them. Thus, it is guaranteed that there is one and only one representative arm for each exploration phase. With the arm counter updating rule specified in Section~\ref{subsec:alg_expl}, the counter of arm $\rho_r$ will certainly increase by $1$ after epoch $r$.
	
	\textbf{Step I: Regret decomposition. }With respect to the representative arm, we decompose the exploration regret as
	\begin{align}
		R_e(T) &= \Eb\left[\sum_r 2^{p_r}(V_{\boldsymbol{\mu},*}-V_{\boldsymbol{\mu},S_r})\right]\notag\\
		& = \Eb\left[\sum_r\sum_{(k,m)\in [K]\times [M]} 2^{p_r}(V_{\boldsymbol{\mu},*}-V_{\boldsymbol{\mu},S_r})\mathds{1}\left\{\rho_r = (k,m)\right\}\right]\notag\\
		& \overset{(a)}{=} \Eb\left[\sum_r\sum_{(k,m)\in [K]\times [M]} 2^{p^r_{k,m}}(V_{\boldsymbol{\mu},*}-V_{\boldsymbol{\mu},S_r})\mathds{1}\left\{\rho_r = (k,m)\right\}\right]\notag\\
		& \overset{(b)}{=} \Eb\left[\sum_r\sum_{(k,m)\in [K]\times [M]} \sum_{n=1}^{N(k,m)}2^{p^r_{k,m}}\Delta^{k,m}_{n}\mathds{1}\left\{\rho_r = (k,m), S_r = S^{k,m}_{n}\right\}\right]\notag\\
		& \overset{(c)}{=} \Eb\left[\sum_{(k,m)\in [K]\times [M]}\sum_{p_{k,m}\geq 0} \sum_{n=1}^{N(k,m)}2^{p_{k,m}}\Delta^{k,m}_{n}\mathds{1}\left\{S_{k,m,p_{k,m}} = S^{k,m}_{n}\right\}\right]\notag\\
		& \overset{(d)}{=} \sum_{(k,m)\in [K]\times [M]}R^{k,m}_e(T),\label{eqn:regret_decomp_general}
	\end{align}
	where equality (a) is from the definition of the representative arm that if $\rho_r = (k,m)$, it holds that $p_r = p^r_{k,m}$. Equality (b) further associates the regret of each exploration phase with specific sub-optimal matchings. $S_{k,m,p_{k,m}}$ denotes the exploration matching with representative arm $(k,m)$ and the corresponding arm counter $p_{k,m}$. Equality (c) holds because once $\rho_r = (k,m)$, its arm counter will increase. Equality (d) denotes $R^{k,m}_e(T) := \Eb\left[\sum_{p_{k,m}>0} \sum_{n=1}^{N(k,m)}2^{p_{k,m}}\Delta^{k,m}_{n}\mathds{1}\left\{S_{k,m,p_{k,m}} = S^{k,m}_{n}\right\}\right]$, which represents the regret associated with arm $(k,m)$.
	
	For term $R_e^{k,m}(T)$, we further have
	\begin{align*}
		R^{k,m}_e(T) =& \Eb\left[\sum_{p_{k,m}\geq 0} \sum_{n=1}^{N(k,m)}2^{p_{k,m}}\Delta^{k,m}_{n}\mathds{1}\left\{S_{k,m,p_{k,m}} = S^{k,m}_{n}\right\}\right]\\
		 =& \sum_{p_{k,m}\geq 0} \sum_{n=1}^{N(k,m)}2^{p_{k,m}}\Delta^{k,m}_{n}\Pb\left(S_{k,m,p_{k,m}} = S^{k,m}_{n}\right)\\
		 \overset{(a)}{\leq}&  \sum_{p_{k,m}>0} \sum_{n=1}^{N(k,m)}2^{p_{k,m}}\Delta^{k,m}_{n}\Pb\left(S_{k,m,p_{k,m}} = S^{k,m}_{n}|\Ec_{k,m,p_{k,m}}\right)\Pb\left(\Ec_{k,m,p_{k,m}}\right)\\
		&+\sum_{p_{k,m}\geq 0} \sum_{n=1}^{N(k,m)}2^{p_{k,m}}\Delta^{k,m}_{n}\Pb\left(S_{k,m,p_{k,m}} = S^{k,m}_{n}|\bar{\Ec}_{k,m,p_{k,m}}\right)\Pb\left(\bar{\Ec}_{k,m,p_{k,m}}\right)\\
		\leq&  \sum_{p_{k,m}\geq0} \sum_{n=1}^{N(k,m)}2^{p_{k,m}}\Delta^{k,m}_{n}\Pb\left(S_{k,m,p_{k,m}} = S^{k,m}_{n}|\Ec_{k,m,p_{k,m}}\right)\\
		&+\sum_{p_{k,m}\geq0} 2^{p_{k,m}}\Delta^{k,m}_{\max}\Pb\left(\bar{\Ec}_{k,m,p_{k,m}}\right)\\
		\leq&  \underbrace{\sum_{h=0}^{N(k,m)}\sum_{q_h^{k,m}\leq p_{k,m}< q_{h+1}^{k,m}} \sum_{n=1}^{N(k,m)}2^{p_{k,m}}\Delta^{k,m}_{n}\Pb\left(S_{k,m,p_{k,m}} = S^{k,m}_{n}|\Ec_{k,m,p_{k,m}}\right)}_{\text{term (A)}}\\
		&+\underbrace{\sum_{p_{k,m}\geq0} 2^{p_{k,m}}\Delta^{k,m}_{\max}\Pb\left(\bar{\Ec}_{k,m,p_{k,m}}\right)}_{\text{term (B)}},
	\end{align*}
	where equality (a) introduces the notion of the ``nice event'' $\Ec_{k,m,p_{k,m}}$, which is described in the following.
	
	At epoch $r$, the nice event $\Ec_{r}$ is defined as 
	\begin{align*}
		\Ec_r = \left\{\forall (k,m)\in [K]\times [M], -\sqrt{\frac{3\ln t_r}{2^{p^r_{k,m}+1}}}< \tilde{\mu}^r_{k,m}- \mu_{k,m} < \sqrt{\frac{3\ln t_r}{2^{p^r_{k,m}+1}}}+\sqrt{\frac{1}{2^{p^r_{k,m}}}} \right\}.
	\end{align*}
	
	Furthermore, when the representative arm in epoch $r$ is arm $(k,m)$ with counter $p_{k,m}$, $\Ec_r$ is denoted as $\Ec_{k,m,p_{k,m}}$.
	
	\textbf{Step II: Bounding term (B).} We start with term (B) by bounding the probability that event $\bar{\Ec}_{r}$ happens. Specifically, it holds that
	\begin{align}
		\Pb\left(\bar{\Ec}_{r}\right) 
		&\leq \sum_{(k,m)\in [K]\times [M]} \Pb\left(\tilde{\mu}^r_{k,m}- \mu_{k,m}\leq -\sqrt{\frac{3\ln t_r}{2^{p^r_{k,m}+1}}}\right) \notag\\
		&+ \sum_{(k,m)\in [K]\times [M]}\Pb\left(\tilde{\mu}^r_{k,m}- \mu_{k,m}\geq  \sqrt{\frac{3\ln t_r}{2^{p^r_{k,m}+1}}}+\sqrt{\frac{1}{2^{p^r_{k,m}}}} \right)\notag\\
		&= \sum_{(k,m)\in [K]\times [M]} \Pb\left(\tilde{\mu}^r_{k,m}- \hat{\mu}^r_{k,m}+\hat{\mu}^r_{k,m}-\mu_{k,m} \leq -\sqrt{\frac{3\ln t_r}{2^{p^r_{k,m}+1}}} \right)\notag\\
		&+\sum_{(k,m)\in [K]\times [M]} \Pb\left( \tilde{\mu}^r_{k,m}- \hat{\mu}^r_{k,m}+\hat{\mu}^r_{k,m}-\mu_{k,m} \geq \sqrt{\frac{3\ln t_r}{2^{p^r_{k,m}+1}}}+\sqrt{\frac{1}{2^{p^r_{k,m}}}} \right)\notag\\
		&\overset{(a)}{\leq} \sum_{(k,m)\in [K]\times [M]} \Pb\left(\hat{\mu}^r_{k,m}-\mu_{k,m} \leq -\sqrt{\frac{3\ln t_r}{2^{p^r_{k,m}+1}}} \right)\notag\\
		&+\sum_{(k,m)\in [K]\times [M]} \Pb\left( \hat{\mu}^r_{k,m}-\mu_{k,m} \geq \sqrt{\frac{3\ln t_r}{2^{p^r_{k,m}+1}}} \right)\notag\\
		&\leq \sum_{(k,m)\in [K]\times [M]}\sum_{p_{k,m}= 0}^{\lfloor\log_2(t_r)\rfloor} 2\Pb\left(\hat{\mu}^r_{k,m}-\mu_{k,m} \geq \sqrt{\frac{3\ln t_r}{2^{p^r_{k,m}+1}}},p^r_{k,m} = p_{k,m}\right)\notag\\
		&\leq \sum_{(k,m)\in [K]\times [M]}\sum_{p_{k,m}= 0}^{\lfloor\log_2(t_r)\rfloor} 2\Pb\left(\frac{\sum_{\tau = 1}^{2^{p_{k,m}}}\gamma^{k,m}_\tau}{2^{p_{k,m}}}-\mu_{k,m} \geq \sqrt{\frac{3\ln t_r}{2^{p_{k,m}+1}}}\right)\notag\\
		&\overset{(b)}{\leq} \sum_{(k,m)\in [K]\times [M]}\sum_{p_{k,m}= 0}^{\lfloor\log_2(t_r)\rfloor} 2\exp\left[-2\cdot2^{p_{k,m}}\frac{3\ln t_r}{2^{p_{k,m}+1}}\right]\notag\\
		& \leq  2KM\frac{\lfloor\log_2(t_r)\rfloor+1}{(t_r)^3}\notag\\
		& \leq 2KM\frac{1}{(t_r)^2}\notag\\
		& \overset{(c)}{\leq} 2KM\frac{1}{(2^{p_r})^2},\label{eqn:nice_event}
	\end{align}
    where inequality (a) holds because $\tilde{\mu}^r_{k,m}=\texttt{ceil}(\hat{\mu}^r_{k,m})$ with $\lceil 1 +p^r_{k,m}/2\rceil$ bits and $\tilde{\mu}^r_{k,m}-\hat{\mu}^r_{k,m}>0$ Inequality (b) is from the Hoeffding's inequality. Inequality (c) utilizes the observation that $t_r\geq 2^{p_r}$.
	
	With Eqn.~\eqref{eqn:nice_event}, we can further bound term (B) as
	\begin{align*}
		\text{term (B)} = &\sum_{p_{k,m}\geq0} 2^{p_{k,m}}\Delta^{k,m}_{\max}\Pb\left(\bar{\Ec}_{k,m,p_{k,m}}\right)\\
		\overset{(a)}{\leq}& 2\sum_{p_{k,m}\geq0} 2^{p_{k,m}}\Delta^{k,m}_{\max}\cdot KM\frac{1}{(2^{p_{k,m}})^2}\\
		=& 2\sum_{p_{k,m}\geq 0} \Delta^{k,m}_{\max}\cdot KM\frac{1}{2^{p_{k,m}}}\\
		\leq & 4KM\Delta^{k,m}_{\max},
	\end{align*}
	where inequality (a) is with Eqn.~\eqref{eqn:nice_event} and $p_r = p_{k,m}$. 
	
	\textbf{Step III: Bounding term (A).} Before bounding term (A), we first establish the following implications. For epoch $r$, if $\rho_r = (k,m)$ and $p_r = p^r_{k,m} = p_{k,m}$, denoting $\boldsymbol{\bar{\mu}}_r$ and $S_r$ as $\boldsymbol{\bar{\mu}}^{k,m,p_{k,m}}$ and $S_{k,m,p_{k,m}}$ respectively, if event $\Ec_{k,m,p_{k,m}}$ happens, we have
	\begin{align}
		&\text{$p_{k,m}\geq q^{k,m}_h$, the oracle outputs $S_{k,m,p_{k,m}} = S_n^{k,m}$}\notag\\
		\Rightarrow & p_{k,m}\geq q^{k,m}_h,\forall S\in\Sc_*\backslash\Sc_c, v(\boldsymbol{\bar{\mu}}^{k,m,p_{k,m}}_{S^{k,m}_{n}}\odot \boldsymbol{\eta}_{S^{k,m}_{n}})\geq v(\boldsymbol{\bar{\mu}}^{k,m,p_{k,m}}_{S}\odot \boldsymbol{\eta}_{S})\notag\\
		\Rightarrow & p_{k,m}\geq q^{k,m}_h,\forall S\in\Sc_*\backslash\Sc_c, v(\boldsymbol{\bar{\mu}}^{k,m,p_{k,m}}_{S^{k,m}_{n}})\geq v(\boldsymbol{\bar{\mu}}^{k,m,p_{k,m}}_{S})\notag\\
		\overset{(a)}{\Rightarrow} & p_{k,m}\geq q^{k,m}_h,\forall S\in\Sc_*\backslash\Sc_c, v(\boldsymbol{\mu}_{S^{k,m}_{n}})+ f\left(\left\|\boldsymbol{\bar{\mu}}^{k,m,p_{k,m}}_{S^{k,m}_{n}}-\boldsymbol{\mu}_{S^{k,m}_{n}}\right\|_{\infty}\right) \geq v(\boldsymbol{\bar{\mu}}^{k,m,p_{k,m}}_{S})\notag\\
		\overset{(b)}{\Rightarrow}& p_{k,m}\geq q^{k,m}_h,\forall S\in\Sc_*\backslash\Sc_c,  V_{\boldsymbol{\mu},S^{k,m}_{n}}+f\left(2\sqrt{\frac{3\ln t_r}{2^{p_{k,m}+1}}}+\sqrt{\frac{1}{2^{p_{k,m}}}}\right)\geq V_{\boldsymbol{\mu},*}\notag\\
		\overset{(c)}{\Rightarrow}& p_{k,m}\geq q^{k,m}_h,V_{S^{k,m}_{n}}+\Delta^{k,m}_h>V_*,\label{eqn:contra}
	\end{align}
	where implication (a) is from Assumption~\ref{asp:bound} and implication (b) utilizes the definition of $\Ec_{k,m,p_{k,m}}$, Assumption~\ref{asp:mono} and that arms in $S_{k,m,p_{k,m}}$ have counters at least $p_{k,m}$. Implication (c) is from the definition of $q^{k,m}_h$ and Eqn.~\eqref{eqn:qkm}.
	
	With Eqn.~\eqref{eqn:contra}, we can get that if $p_{k,m}\geq q^{k,m}_h$, the matchings $S^{k,m}_n$ with $n\leq h$ cannot be $S_r$; otherwise it contradicts with the definition of $\Delta^{k,m}_h$. Thus, we can further bound term (A) as
	\begin{align*}
		\text{term (A)} =& \sum_{h=0}^{N(k,m)}\sum_{q_h^{k,m}\leq p_{k,m}< q_{h+1}^{k,m}} \sum_{n=1}^{N(k,m)}2^{p_{k,m}}\Delta^{k,m}_{n}\Pb\left(S_{k,m,p_{k,m}} = S^{k,m}_{n}|\Ec_{k,m,p_{k,m}}\right)\\
		=& \sum_{h=0}^{N(k,m)}\sum_{q_h^{k,m}\leq p_{k,m}< q_{h+1}^{k,m}} \sum_{n=h+1}^{N(k,m)}2^{p_{k,m}}\Delta^{k,m}_{n}\Pb\left(S_{k,m,p_{k,m}} = S^{k,m}_{n}|\Ec_{k,m,p_{k,m}}\right)\\
		\overset{(a)}{\leq} & \sum_{h=0}^{N(k,m)}\sum_{q_h^{k,m}\leq p_{k,m}< q_{h+1}^{k,m}} \sum_{n=h+1}^{N(k,m)}2^{p_{k,m}}\Delta^{k,m}_{h+1}\Pb\left(S_{k,m,p_{k,m}} = S^{k,m}_{n}|\Ec_{k,m,p_{k,m}}\right)\\
		\overset{(b)}{\leq} &\sum_{h=0}^{N(k,m)}\sum_{q_h^{k,m}\leq p_{k,m}< q_{h+1}^{k,m}} 2^{p_{k,m}}\Delta^{k,m}_{h+1}\\
		 = & \sum_{h=0}^{N(k,m)} (2^{q^{k,m}_{h+1}}-2^{q^{k,m}_{h}})\Delta^{k,m}_{h+1}\\
		 = & \sum_{h=0}^{N(k,m)-1} (2^{q^{k,m}_{h+1}}-2^{q^{k,m}_{h}})\Delta^{k,m}_{h+1}\\
		  \leq &2^{q^{k,m}_{N(k,m)}}\Delta^{k,m}_{N(k,m)}+ \sum_{h=1}^{N(k,m)-1} 2^{q^{k,m}_{h}}\left(\Delta^{k,m}_h-\Delta^{k,m}_{h+1}\right)\\
		  \overset{(c)}{\leq} & \frac{28\Delta^{k,m}_{N(k,m)}\ln(T)}{(f^{-1}(\Delta^{k,m}_{N(k,m)}))^2}+ \sum_{h=1}^{N(k,m)-1} \frac{28\ln(T)}{(f^{-1}(\Delta^{k,m}_{h}))^2}\left(\Delta^{k,m}_h-\Delta^{k,m}_{h+1}\right)\\
		  \overset{(d)}{\leq} & \frac{28\Delta^{k,m}_{N(k,m)}\ln(T)}{(f^{-1}(\Delta^{k,m}_{N(k,m)}))^2}+ \int_{\Delta^{k,m}_{N(k,m)}}^{\Delta^{k,m}_1} \frac{28\ln(T)}{(f^{-1}(x))^2}dx\\
		  = & \frac{28\Delta^{k,m}_{\min}\ln(T)}{(f^{-1}(\Delta^{k,m}_{\min}))^2}+ \int_{\Delta^{k,m}_{\min}}^{\Delta^{k,m}_{\max}} \frac{28\ln(T)}{(f^{-1}(x))^2}dx,
	\end{align*}
	where inequality (a) holds because $\forall n\geq h+1$, $\Delta^{k,m}_{n}\leq \Delta^{k,m}_{h+1}$, and inequality (b) is from $ \sum_{n=h+1}^{N(k,m)}\Pb\left(S_{k,m,p_{k,m}} = S^{k,m}_{n}|\Ec_{k,m,p_{k,m}}\right)\leq 1$. Inequality (c) is from the definition of $q_n^{k,m}$ and inequality (d) is because $\frac{28\ln(T)}{(f^{-1}(x))^2}$ is strictly decreasing in $[{\Delta^{k,m}_{N(k,m)}},{\Delta^{k,m}_{1}} ]$.
	
	By combining terms (A) and (B), we have
	\begin{align*}
		R^{k,m}_e(T) &\leq \frac{28\Delta^{k,m}_{\min}\ln(T)}{(f^{-1}(\Delta^{k,m}_{\min}))^2}+ \int_{\Delta^{k,m}_{\min}}^{\Delta^{k,m}_{\max}} \frac{28\ln(T)}{(f^{-1}(x))^2}dx+
	4KM\Delta^{k,m}_{\max}\\
		&\leq \frac{28\Delta^{k,m}_{\max}\ln(T)}{(f^{-1}(\Delta^{k,m}_{\min}))^2}+4KM\Delta^{k,m}_{\max}.
	\end{align*}

	Overall, we conclude that
	\begin{align*}
		R_e(T) &= \sum_{(k,m)\in [K]\times [M]}R^{k,m}_e(T) \\
		&\leq \sum_{(k,m)\in [K]\times [M]}\left[\frac{28\Delta^{k,m}_{\min}\ln(T)}{(f^{-1}(\Delta^{k,m}_{\min}))^2}+ \int_{\Delta^{k,m}_{\min}}^{\Delta^{k,m}_{\max}} \frac{28\ln(T)}{(f^{-1}(x))^2}dx+
		4KM\Delta^{k,m}_{\max}\right]\\
		&\leq \sum_{(k,m)\in [K]\times [M]}\frac{28\Delta^{k,m}_{\max}\ln(T)}{(f^{-1}(\Delta^{k,m}_{\min}))^2}+4K^2M^2\Delta_{\max}.
	\end{align*}
 \end{proof}

Theorems~\ref{thm:general} and \ref{thm:general_full} can be proved by combining Lemmas~\ref{lem:comm_regret}, \ref{lem:expl_regret}, and Eqn.~\eqref{eqn:apdx1_Ro}.

\section{Proof for Theorem~\ref{thm:linear}}
A complete version of Theorem~\ref{thm:linear} is first presented in the following.
\begin{theorem}[\textbf{Complete version of Theorem~\ref{thm:linear}}]\label{thm:linear_full}
	With a linear reward function, the regret of BEACON is upper bounded as
	\begin{align*}
		R_{\textup{linear}}(T) &\leq \sum_{(k,m)\in [K]\times [M]}\frac{3727M}{\Delta^{k,m}_{\min}}\ln(T)+8K^2M^3+M^2K\\
		&+ \left(22 M+ 2M\log_2(K)\right)\left[\frac{2MK}{\ln 2}\ln(T)+MK\left(\frac{3M\sqrt{3\ln(T)}}{\sqrt{2}-1}+\frac{8KM^2}{3}\right)\right]\\
		 &= \tilde{O}\left(\sum_{(k,m)\in [K]\times [M]}\frac{M\log(T)}{\Delta^{k,m}_{\min}}+M^2K\log(T)\right)\\
		&=  \tilde{O}\left(\frac{M^2K\log(T)}{\Delta_{\min}}+M^2K\log(T)\right).
	\end{align*}
\end{theorem}
\begin{proof}[Proof for Theorems~\ref{thm:linear} and \ref{thm:linear_full}]
	Similar to the previous proof, the overall regret $R_{\text{linear}}(T)$ can be decomposed into three parts: the exploration regret $R_{e,\text{linear}}(T)$, the communication regret $R_{c,\text{linear}}(T)$, and the other regret $R_{o,\text{linear}}(T)$, i.e.,
	\begin{align*}
		R_{\text{linear}}(T) = R_{e,\text{linear}}(T) + R_{c,\text{linear}}(T) +R_{o,\text{linear}}(T).
	\end{align*}
	The last component can be similarly bounded as
	\begin{align*}
		R_{o,\text{linear}}(T) \leq \left(\frac{K^2M}{K-M}+2K\right)\Delta_{c}+K\Delta_{\max},
	\end{align*}
	The communication regret and exploration regret are bounded Lemmas~\ref{lem:comm_regret_linear} and \ref{lem:expl_regret_linear} that are presented in the subsequent subsections. Putting them all together completes the proof.
\end{proof}

\subsection{Communication Regret}
\begin{lemma}\label{lem:comm_regret_linear}
	For BEACON, under time horizon $T$,  the communication loss $R_{c,\text{linear}}(T)$ is upper bounded as
	\begin{align*}
		R_{c,\text{linear}}(T) &\leq M^2K+ \left(22 M+ 2M\log_2(K)\right)\left[\frac{2MK}{\ln 2}\ln(T)+MK\left(\frac{3M\sqrt{3\ln(T)}}{\sqrt{2}-1}+\frac{8KM^2}{3}\right)\right].
	\end{align*}
\end{lemma}
\begin{proof}[Proof for Lemma~\ref{lem:comm_regret_linear}]
	From the proof for Lemma~\ref{lem:comm_regret}, we can draw the following facts:
	\begin{itemize}
		\item[(i)] For epoch $1$, communicating $\tilde{\delta}^1_{k,m}$ takes 1 time  step;
		\item[(ii)] For epoch $r>1$, if $p^r_{k,m}> p^{r-1}_{k,m}$, $\tilde{\delta}^r_{k,m}$ is communicated and the communication in expectation takes $2+2\times (1+\Eb[L^r_{k,m}])\leq 18$ time steps;
		\item[(iii)] For epoch $r>1$, the communication of the chosen matching and the batch size parameter takes less than $M(3+2\log_2(K))+M$ time steps.
	\end{itemize}
	These facts hold for the general reward functions, thus naturally hold for the linear reward function. 
	
	However, with the linear reward function, the loss caused by communication can be characterized more carefully as
	\begin{align*}
		 &R_{c,\text{linear}}(T) \overset{(a)}{\leq} MK\times M \\
		 &+ \Eb\left[\sum_{r}(2+V_{\boldsymbol{\mu},*}-V_{\boldsymbol{\mu},S_r})\mathds{1}\left\{\Ec_r\right\}\left[\sum_{(k,m)} 18 \mathds{1}\left\{p^r_{k,m}\geq p^{r-1}_{k,m}\right\}+ M(3+2\log_2(K))+M\right]\right]\\
		 &+ \Eb\left[\sum_{r}M\mathds{1}\left\{\bar{\Ec}_r\right\}\left[\sum_{(k,m)} 18 \mathds{1}\left\{p^r_{k,m}\geq p^{r-1}_{k,m}\right\}+ M(3+2\log_2(K))+M\right]\right]\\
		 &\overset{(b)}{\leq} M^2K+ \sum_{r}\Eb\left[(2+V_{\boldsymbol{\mu},*}-V_{\boldsymbol{\mu},S_r})\mathds{1}\left\{\Ec_r\right\}+M\mathds{1}\left\{\bar{\Ec}_r\right\}\right]\left(22 M+ 2M\log_2(K)\right)\\
		 &\overset{(c)}{\leq} M^2K+ \sum_{r}\left(2+3M\sqrt{\frac{3\ln(T)}{2^{p_r+1}}}+2M\frac{KM}{(2^{p_r})^2}\right)\left(22 M+ 2M\log_2(K)\right)\\
		 &\leq M^2K+ \left(22 M+ 2M\log_2(K)\right)\left[2MK\log_2(T)+MK\sum_{p_r=0}^{\lceil\log_2 T\rceil}\left(3M\sqrt{\frac{3\ln(T)}{2^{p_r+1}}}+2\frac{KM^2}{(2^{p_r})^2}\right)\right]\\
		 &\leq M^2K+ \left(22 M+ 2M\log_2(K)\right)\left[2MK\log_2(T)+MK\left(3M\sqrt{3\ln(T)}\frac{1}{\sqrt{2}-1}+\frac{8KM^2}{3}\right)\right]
	\end{align*}
	where inequality (a) is from that there are at most $2$ players colliding with each other (leader and one follower) under the nice event $\Ec_r$. Specifically, with arms in $S_r$ used for communications in epoch $r$, one communication step leads to a loss at most $2+V_{\boldsymbol{\mu},*}-V_{\boldsymbol{\mu},S_r}$. Inequality (b) is from that in each epoch $r>1$, at most $M$ arms statistics need to be communicated. Inequality (c)  holds because if the nice event $\Ec_r$ happens
	\begin{align*}
		&\forall S\in \Sc_*\backslash \Sc_c, v({\boldsymbol{\bar{\mu}}^r_{S_r}})\geq v({\boldsymbol{\bar{\mu}}^r_{S}})\\
		\Rightarrow & \forall S\in \Sc_*\backslash \Sc_c, V_{\boldsymbol{\mu},S_r}+ M\left(2\sqrt{\frac{3\ln t_r}{2^{p_r+1}}}+\sqrt{\frac{1}{2^{p_r}}}\right) \geq v(\boldsymbol{\bar{\mu}}^r_{S_r})\geq v(\boldsymbol{\bar{\mu}}^r_{S})> v(\boldsymbol{\mu}_{S})  = V_{\boldsymbol{\mu},*}\\
		\Rightarrow & V_{\boldsymbol{\mu},*} - V_{\boldsymbol{\mu},S_r} \leq M\left(2\sqrt{\frac{3\ln t_r}{2^{p_r+1}}}+\sqrt{\frac{1}{2^{p_r}}}\right)\leq 3M\sqrt{\frac{3\ln(T)}{2^{p_r+1}}};
	\end{align*}
	otherwise, the nice event does not happen with $\Pb(\bar{\Ec}_r)\leq \frac{2KM}{(2^{p_r})^2}$ proved in the Eqn.~\eqref{eqn:nice_event}, $\Eb[M\mathds{1}\left\{\bar{\Ec}_r\right\}]\leq  2M\frac{KM}{(2^{p_r})^2}$.
\end{proof}

\subsection{Exploration Regret}
\begin{lemma}\label{lem:expl_regret_linear}
	For BEACON, under time horizon $T$,  the exploration loss $R_{e,\text{linear}}(T)$ is upper bounded as
	\begin{align*}
		R_{e,\textup{linear}}(T) &\leq \sum_{(k,m)}\frac{3727M}{\Delta^{k,m}_{\min}}\ln(T)+4K^2M^2\Delta_{\max}.
	\end{align*}
\end{lemma}
\begin{proof}[Proof for Lemma~\ref{lem:expl_regret_linear}]
	The following proof is based on the proof for CUCB with a linear reward function in \cite{kveton2015tight}, but is carefully designed for the complicated batched exploration. In the following proof, we introduce the following notations:
	\begin{align*}
		 &S^*=[s_1^*,...,s_M^*]\in \Sc_*\backslash\Sc_c \text{: one particular collision-free optimal matching}; \\
		 &\Delta_{S_r}:=V_{\boldsymbol{\mu},*}-V_{\boldsymbol{\mu},S_r};\\
		 &[\tilde{M}_r]:=\{m|m\in[M], s^r_m\neq s^*_m\}.
	\end{align*}

	\textbf{Step I: Regret decomposition. } First, we can decompose the exploration regret $R_{e,\text{linear}}(T)$ as
	\begin{align}
		R_{e,\text{linear}}(T) &= \Eb\left[\sum_{r} 2^{p_r} (V_{\boldsymbol{\mu},*}-V_{\boldsymbol{\mu},S_r})\right]\notag\\
			& = \Eb\left[\sum_{r} 2^{p_r} \Delta_{S_r}\mathds{1}\left\{\Ec_r, \Delta_{S_r}>0\right\}\right]+\Eb\left[\sum_{r} 2^{p_r} \Delta_{S_r}\mathds{1}\left\{\bar{\Ec}_r, \Delta_{S_r}>0\right\}\right]\label{eqn:regret_linear_decomp}\\
			&\overset{(a)}{\leq} \underbrace{\Eb\left[\sum_{r} 2^{p_r} \Delta_{S_r}\mathds{1}\left\{\sum_{m\in[\tilde{M}_r]}\left(2\sqrt{\frac{3\ln t_r}{2^{p^r_{s^r_m,m}+1}}}+\sqrt{\frac{1}{2^{p^r_{s^r_m,m}}}}\right) \geq \Delta_{S_r}, \Delta_{S_r}>0\right\}\right]}_{\text{term (C)}}\notag\\
			&+\underbrace{\Eb\left[\sum_{r} 2^{p_r} \Delta_{S_r}\mathds{1}\left\{\bar{\Ec}_r\right\}\right]}_{\text{term (D)}},\notag
	\end{align}
	where inequality (a) is because when the nice event $\Ec_r$ happens, choosing a sub-optimal matching $S_r$, i.e., $\Delta_{S_r}>0$, implies
	\begin{align*}
		&\forall S\in\Sc_*, v(\boldsymbol{\bar{\mu}}^r_{S_r}) \geq v(\boldsymbol{\bar{\mu}}^r_{S})\\
		\Rightarrow &v(\boldsymbol{\bar{\mu}}^r_{S_r}) \geq v(\boldsymbol{\bar{\mu}}^r_{S^*})\\
		\Rightarrow &\sum_{m\in[\tilde{M}_r]}\bar{\mu}^r_{s^r_m,m} \geq \sum_{m\in[\tilde{M}_r]}\bar{\mu}^r_{s^*_m,m}\\
		\Rightarrow & \sum_{m\in[\tilde{M}_r]}\mu_{s^r_m,m}+ \sum_{m\in[\tilde{M}_r]}\left(2\sqrt{\frac{3\ln t_r}{2^{p^r_{s^r_m,m}+1}}}+\sqrt{\frac{1}{2^{p^r_{s^r_m,m}}}}\right)\geq\sum_{m\in[\tilde{M}_r]}\mu_{s^*_m,m}\\
		\Rightarrow & \sum_{m\in[\tilde{M}_r]}\left(2\sqrt{\frac{3\ln t_r}{2^{p^r_{s^r_m,m}+1}}}+\sqrt{\frac{1}{2^{p^r_{s^r_m,m}}}}\right) \geq V_{\boldsymbol{\mu},*}- V_{\boldsymbol{\mu},S_r}= \Delta_{S_r}.
	\end{align*}

	\textbf{Step II: Bounding term (D). }
	With essentially the same approach of bounding term (B) in the proof of Lemma~\ref{lem:expl_regret}, especially Eqn.~\eqref{eqn:nice_event},  we can directly bound term (D) as
	\begin{align*}
		\text{term (D)} = \Eb\left[\sum_{r} 2^{p_r} \Delta_{S_r}\mathds{1}\left\{\bar{\Ec}_r\right\}\right]\leq 4K^2M^2\Delta_{\max}.
	\end{align*}
	
	\textbf{Step III: Bounding term (C).} First, we denote event
	\begin{align*}
		\Fc_r = \left\{\sum_{m\in[\tilde{M}_r]}\left(2\sqrt{\frac{3\ln t_r}{2^{p^r_{s^r_m,m}+1}}}+\sqrt{\frac{1}{2^{p^r_{s^r_m,m}}}}\right) \geq \Delta_{S_r}, \Delta_{S_r}>0\right\},
	\end{align*}
	thus 
	\begin{align*}
		\text{term (C)} &= \Eb\left[\sum_{r} 2^{p_r} \Delta_{S_r}\mathds{1}\left\{\sum_{m\in[\tilde{M}_r]}\left(2\sqrt{\frac{3\ln t_r}{2^{p^r_{s^r_m,m}+1}}}+\sqrt{\frac{1}{2^{p^r_{s^r_m,m}}}}\right) \geq \Delta_{S_r}, \Delta_{S^r}>0\right\}\right]\\
		&= \Eb\left[\sum_{r} 2^{p_r} \Delta_{S_r}\mathds{1}\left\{\Fc_r\right\}\right].
	\end{align*}

	Following the ideas in \cite{kveton2015tight}, we introduce two decreasing sequences of constants:
	\begin{align*}
		1 = b_0>&b_1>b_2>\cdots>b_i>\cdots\\
		  & a_1>a_2>\cdots >a_i >\cdots
	\end{align*}
	such that $\lim_{i\to\infty}a_i = \lim_{i\to\infty}b_i = 0$. Furthermore, we specify
    $q_{i,S_r}$ as the integer satisfying
	\begin{align*}
		2^{q_{i,S_r}-1}\leq a_i\frac{M^2}{(\Delta_{S_r})^2}\ln(T) < 2^{q_{i,S_r}} \leq 2a_i\frac{M^2}{(\Delta_{S_r})^2}\ln(T).
	\end{align*}
	For convenience, we denote $q_{0,S_r} = 0$ and $q_{\infty,S_r} = \infty$. Also, set $H^r_{i}$ is defined as
	\begin{align*}
		\forall i\geq 1, H^r_i = \left\{m|m\in[\tilde{M}_r], p^r_{s^r_m,m}< q_{i,S_r}\right\},
	\end{align*}
    which represents the arms that are not sufficiently sampled compared with $q_{i,S_r}$, and $H^r_0:=[\tilde{M}_r]$.
    
	With the above introduce notations, we define the following infinitely-many events at epoch $r$ as
	\begin{align*}
		G^r_{1}&=\left\{\left|H^r_1\right|\geq b_1 M \right\};\\
		G^r_{2}&=\left\{\left|H^r_1\right|<b_1 M \right\}\cap \left\{\left|H^r_2\right|\geq b_2 M \right\};\\
		&\cdots\\
		G^r_{i}&=\left\{\left|H^r_1\right|<b_1 M \right\}\cap \left\{\left|H^r_2\right|< b_2 M \right\}\cap \cdots \cap \left\{\left|H^{r}_{i-1}\right|< b_{i-1} M \right\}\cap \left\{\left|H^r_i\right|\geq b_{i} M \right\};\\
		&\cdots
	\end{align*}
	Clearly, these events are mutually exclusive. We have the following proposition.
	\begin{proposition}\label{prop:split}
		Let
		\begin{align}\label{eqn:alpha_beta}
			\sqrt{14} \sum_{i=1}^{\infty}\frac{b_{i-1}-b_i}{\sqrt{a_i}}\leq 1.
		\end{align}
		If event $\Fc_r$ happens at epoch $r$, then there exists $i$ such that $G^r_i$ happens. 
	\end{proposition}
	This proposition can be proved by assuming that $\Fc_r$ happens while none of $G^r_i$ happens. Denoting $\bar{G}_r = \overline{\cup_i G^r_i}$, we can get
	\begin{align*}
		\bar{G}_r &= \overline{\cup_{i=1}^{\infty} G^r_i} \\
		& = \cap_{i=1}^{\infty} \bar{G}^r_i\\
		& = \cap_{i=1}^{\infty} \left[\overline{\left(\cap_{j=1}^{i-1}\left\{\left|H^{r}_{j}\right|< b_{j} M \right\}\right)}\cup \overline{\left\{\left|H^{r}_{i}\right|\geq b_{i} M \right\}}\right]\\
		& = \cap_{i=1}^{\infty} \left[\left(\cup_{j=1}^{i-1}\overline{\left\{\left|H^{r}_{j}\right|< b_{j} M \right\}}\right)\cup \overline{\left\{\left|H^{r}_{i}\right|\geq  b_{i} M \right\}}\right]\\
		& = \cap_{i=1}^{\infty} \left[\left(\cup_{j=1}^{i-1}\left\{\left|H^{r}_{j}\right|\geq b_{j} M \right\}\right)\cup \left\{\left|H^{r}_{i}\right|<  b_{i} M \right\}\right]\\
		& = \cap_{i=1}^{\infty} \left\{\left|H^{r}_{i}\right|<  b_{i} M \right\}.
	\end{align*}
	If $\bar{G}_r$ happens, denoting $\tilde{H}^r_i = [\tilde{M}_r]\backslash H^r_i$, which implies $\tilde{H}^r_{i-1}\subseteq \tilde{H}^r_{i}$  and $[\tilde{M}_r] = \cup_i(\tilde{H}_{i}^r\backslash\tilde{H}_{i-1}^r)$, then it holds that
	\begin{align*}
		&\sum_{m\in[\tilde{M}_r]}\left(2\sqrt{\frac{3\ln T}{2^{p^r_{s^r_m,m}+1}}}+\sqrt{\frac{1}{2^{p^r_{s^r_m,m}}}}\right)\\
		\leq & 3\sqrt{3\ln T}\sum_{m\in[\tilde{M}_r]}\frac{1}{\sqrt{2^{p^r_{s^r_m,m}+1}}}\\
		 = & 3\sqrt{3\ln T}\sum_{i=1}^{\infty}\sum_{m\in \tilde{H}_{i}^r\backslash \tilde{H}_{i-1}^r}\frac{1}{\sqrt{2^{p^r_{s^r_m,m}+1}}}\\
		 = & 3\sqrt{3\ln T}\sum_{i=1}^{\infty}\frac{|\tilde{H}_{i}^r\backslash \tilde{H}_{i-1}^r|}{\sqrt{2^{q_{i-1,S_r}+1}}}\\
		 \leq & 3\sqrt{3\ln T}\sum_{i=1}^{\infty}\frac{|\tilde{H}_{i}^r\backslash \tilde{H}_{i-1}^r|}{\sqrt{2a_i\frac{M^2}{(\Delta_{S_r})^2}\ln(T)}}\\
		 \leq & 3\sqrt{3/2} \frac{\Delta_{S_r}}{M}\sum_{i=1}^{\infty}\left(|H^r_{i-1}|-|H^r_{i}|\right)\frac{1}{\sqrt{a_i}}\\
		  = & 3\sqrt{3/2} \frac{\Delta_{S_r}}{M}|H^r_0|\frac{1}{\sqrt{a_1}} + 3\sqrt{3/2} \frac{\Delta_{S_r}}{M}\sum_{i=1}^{\infty}|H^r_i|\left(\frac{1}{\sqrt{a_{i+1}}}-\frac{1}{\sqrt{a_{i}}}\right)\\
		  \overset{(a)}{\leq} & 3\sqrt{3/2} \frac{\Delta_{S_r}}{M}b_0M\frac{1}{\sqrt{a_1}} + 3\sqrt{3/2} \frac{\Delta_{S_r}}{M}\sum_{i=1}^{\infty}b_i M\left(\frac{1}{\sqrt{a_{i+1}}}-\frac{1}{\sqrt{a_{i}}}\right)\\
		   < & \sqrt{14} \sum_{i=1}^{\infty}\frac{b_{i-1}-b_i}{\sqrt{a_i}} \Delta_{S_r}\\
		   \leq & \Delta_{S_r},
	\end{align*}
	where inequality is because $|H^{r}_{i}<  b_{i} M$ with $\bar{G}_r$ happening. This result contradicts with the definition of $\Fc_r$ as
	\begin{align*}
	    \Fc_r = \left\{\sum_{m\in[\tilde{M}^r]}\left(2\sqrt{\frac{3\ln t_r}{2^{p^r_{s^r_m,m}+1}}}+\sqrt{\frac{1}{2^{p^r_{s^r_m,m}}}}\right) \geq \Delta_{S_r}, \Delta_{S_r}>0\right\}.
	\end{align*}
	
	With Proposition~\ref{prop:split}, when Eqn.~\eqref{eqn:alpha_beta} holds, we can further decompose term (C) as
	\begin{align*}
		\text{term (C)}  &= \Eb\left[\sum_{r} 2^{p_r} \Delta_{S_r}\mathds{1}\left\{\Fc_r\right\}\right] = \Eb\left[\sum_{r}\sum_{i=1}^{\infty} 2^{p_r} \Delta_{S_r}\mathds{1}\left\{G^r_i,\Delta_{S_r}>0\right\}\right].
	\end{align*}
	Then, the following events are defined
	\begin{align*}
		G^r_{i,k,m} = G^r_{i}\cap\left\{m\in [\tilde{M}_r],s^r_m=k, p^r_{k,m}<q_{i,S_r}\right\},
	\end{align*}
	which imply that
	\begin{align*}
		\mathds{1}\left\{G^r_i,\Delta_{S_r}>0\right\}\leq \frac{1}{b_i M}\sum_{(k,m)}\mathds{1}\left\{G^r_{i,s^r_m,m},\Delta_{S_r}>0\right\}
	\end{align*}
	since at least $b_i M$ arms with event $G^r_{i,k,m}$ happening are required to make $G^r_i$ happen.
	
	Thus, recall $\Sc^{k,m}_{b} = \{S|S\in \Sc_b, s_m=k\}= \{S^{k,m}_{1},..., S^{k,m}_{N(k,m)}\}$, we can get
	\begin{align*}
		&\text{term (C)}   = \Eb\left[\sum_{r}\sum_{i=1}^{\infty} 2^{p_r} \Delta_{S_r}\mathds{1}\left\{G^r_i,\Delta_{S_r}>0\right\}\right]\\
		&\leq \Eb\left[\sum_{r}\sum_{i=1}^{\infty} 2^{p_r} \Delta_{S_r}\frac{1}{b_i M}\sum_{(k,m)}\mathds{1}\left\{G^r_{i,k,m},\Delta_{S_r}>0\right\}\right]\\
		& \leq \Eb\left[\sum_{r}\sum_{i=1}^{\infty} 2^{p_r} \Delta_{S_r}\frac{1}{b_i M}\sum_{(k,m)}\mathds{1}\left\{m\in[\tilde{M}_r],s^r_m=k, p^r_{k,m}<q_{i,S_r},\Delta_{S_r}>0\right\}\right]\\
		& = \Eb\left[\sum_{(k,m)}\sum_{n=1}^{N(k,m)}\sum_{r}\sum_{i=1}^{\infty} 2^{p_r} \frac{1}{b_i M}\mathds{1}\left\{s^r_m=k, p^r_{k,m}<q_{i,S^{k,m}_n},S_r=S^{k,m}_n\right\}\Delta^{k,m}_n\right]\\
		& = \Eb\left[\sum_{(k,m)}\sum_{i=1}^{\infty}\underbrace{\sum_{r} \sum_{n=1}^{N(k,m)}2^{p_r} \frac{1}{b_i M}\mathds{1}\left\{s^r_m=k, p^r_{k,m}<q_{i,S^{k,m}_n},S_r=S^{k,m}_n\right\}\Delta^{k,m}_n}_{\text{term (E)}}\right]\\
		& \overset{(a)}{\leq } \Eb\left[\sum_{(k,m)}\left[\sum_{i=1}^{\infty}\frac{6a_i }{b_i}\right]\frac{M}{\Delta^{k,m}_{N(k,m)}}\ln(T)\right]
	\end{align*}
	where inequality (a) holds because term (E) can be bounded as
	\begin{align*}
		\text{term (E)}&=\sum_{r} \sum_{n=1}^{N(k,m)}2^{p_r} \frac{1}{b_i M}\mathds{1}\left\{s^r_m=k, p^r_{k,m}<q_{i,S^{k,m}_n},S_r=S^{k,m}_n\right\}\Delta^{k,m}_n\\
		&\leq 3\times 2^{q_{i,S^{k,m}_1}-1}\frac{\Delta^{k,m}_1}{b_iM}+\frac{1}{b_iM}\sum_{n=2}^{N(k,m)}\left(3\times 2^{q_{i,S^{k,m}_n}-1}-3\times 2^{q_{i,S^{k,m}_{n-1}}-1}\right)\Delta^{k,m}_n\\
		&\leq \frac{3a_i M}{b_i\Delta_{1}^{k,m}}\ln(T)+\frac{3a_i M}{b_i}\sum_{n=2}^{N(k,m)}\left(\frac{1}{(\Delta_{n}^{k,m})^2}-\frac{1}{(\Delta_{n-1}^{k,m})^2}\right)\Delta^{k,m}_n\ln(T)\\
		&= \frac{3a_i M}{b_i}\ln(T)\left[\sum_{n=1}^{N(k,m)-1}\frac{\Delta_n^{k,m}-\Delta^{k,m}_{n+1}}{(\Delta^{k,m}_n)^2}+\frac{1}{\Delta^{k,m}_{N(k,m)}}\right]\\
		&\leq \frac{3a_i M}{b_i}\ln(T)\left[\sum_{n=1}^{N(k,m)-1}\frac{\Delta_n^{k,m}-\Delta^{k,m}_{n+1}}{\Delta^{k,m}_n \Delta^{k,m}_{n+1}}+\frac{1}{\Delta^{k,m}_{N(k,m)}}\right]\\
		&\leq \frac{3a_i M}{b_i}\ln(T)\frac{2}{\Delta^{k,m}_{N(k,m)}}.
	\end{align*}
	At last, we specify the choices of $a_i$ and $b_i$, which resolve to the following optimization problem:
	\begin{align*}
		\text{minimize }&\sum_{i=1}^{\infty}\frac{6a_i }{b_i}\\
		\text{subject to }& \lim_{i\to\infty}a_i = \lim_{i\to\infty}b_i = 0\\
		&\text{Monotonicity: }1 = b_0>b_1>b_2>\cdots>b_i>\cdots; a_1>a_2>\cdots >a_i >\cdots\\
		&\text{Eqn.~\eqref{eqn:alpha_beta}: }\sqrt{14} \sum_{i=1}^{\infty}\frac{b_{i-1}-b_i}{\sqrt{a_i}}\leq 1.
	\end{align*}
	We choose $a_i$ and $b_i$ to be geometric sequences as in \cite{kveton2015tight}, specifically $a_i = d(a)^i$ and $b_i = (b)^i$ with $0<a, b<1$ and $d> 0$. Moreover, if $b\leq \sqrt{a}$, to meed Eqn.~\eqref{eqn:alpha_beta}, it needs
	\begin{align*}
		\sqrt{14} \sum_{i=1}^{\infty}\frac{b_{i-1}-b_i}{\sqrt{a_i}} = \sqrt{14} \sum_{i=1}^{\infty}\frac{(b)^{i-1}-(b)^i}{\sqrt{d(a)^i}} = \sqrt{\frac{14}{d}}\frac{1-b}{\sqrt{a}-b}\leq 1 \Rightarrow d\geq  14\left(\frac{1-b}{\sqrt{a}-b}\right)^2.
	\end{align*}
	Thus, the best choice for $d$ is $d = 14\left(\frac{1-b}{\sqrt{a}-b}\right)^2$ and the problem is reformulated as
	\begin{align*}
		\text{minimize }&\sum_{i=1}^{\infty}\frac{6a_i }{b_i} = 84\left(\frac{1-b}{\sqrt{a}-b}\right)^2\frac{\alpha}{b-a}\\
		\text{conditioned on }& 0<a<b<\sqrt{a}<1.
	\end{align*}
	With numerically calculated $a = 0.1459$ and $b = 0.2360$ in \cite{kveton2015tight}, we get $\sum_{i=1}^{\infty}\frac{6a_i }{b_i}\leq 3727$. Thus, we conclude that
	\begin{align*}
		\text{term (C)}  &\leq \Eb\left[\sum_{(k,m)}\left[\sum_{i=1}^{\infty}\frac{6a_i }{b_i}\right]\frac{M}{\Delta^{k,m}_{N(k,m)}}\ln(T)\right]\\
		&\leq \sum_{(k,m)}\frac{3727M}{\Delta^{k,m}_{N(k,m)}}\ln(T) \\
		&\leq \sum_{(k,m)}\frac{3727M}{\Delta^{k,m}_{\min}}\ln(T).
	\end{align*}
	Lemma~\ref{lem:expl_regret_linear} can be proved by combining term (C) and term (D).
\end{proof}

\section{Proof for Theorem~\ref{thm:linear_gapfree}}

\begin{proof}
This proof follows naturally from Theorem~\ref{thm:linear_full} by categorizing sub-optimal gaps with  a threshold $\epsilon$.

Specifically, we can modify Eqn.~\eqref{eqn:regret_linear_decomp} as
\begin{align*}
    R_{e,\text{linear}}(T) &= \Eb\left[\sum_{r} 2^{p_r} (V_{\boldsymbol{\mu},*}-V_{\boldsymbol{\mu},S_r})\right]\notag\\
			& = \Eb\left[\sum_{r} 2^{p_r} \Delta_{S_r}\mathds{1}\left\{\Ec_r, \Delta_{S_r}>0\right\}\right]+\Eb\left[\sum_{r} 2^{p_r} \Delta_{S_r}\mathds{1}\left\{\bar{\Ec}_r, \Delta_{S_r}>0\right\}\right]\\
			&\leq T\epsilon+\Eb\left[\sum_{r} 2^{p_r} \Delta_{S_r}\mathds{1}\left\{\Ec_r, \Delta_{S_r}>\epsilon\right\}\right]+\Eb\left[\sum_{r} 2^{p_r} \Delta_{S_r}\mathds{1}\left\{\bar{\Ec}_r, \Delta_{S_r}>\epsilon\right\}\right]\\
			&\overset{(a)}{\leq} T\epsilon+ \sum_{(k,m)}\frac{3727M}{\epsilon}\ln(T)+4K^2M^2\Delta_{\max},
\end{align*}
where inequality (a) follows the same proof for Lemma \ref{lem:expl_regret_linear}.
For the overall regret, we can further get
\begin{align*}
    R_{\text{linear}}(T) 
    \leq &  T\epsilon+\frac{3727M^2K}{\epsilon}\ln(T)+\text{terms of order $O(\ln(T))$ and independent with $\epsilon$}\\
    \overset{(a)}{\leq} & 124M\sqrt{KT\ln(T)} +\text{terms of order $O(\ln(T))$ and independent with $\epsilon$}\\
     = & O\left(M\sqrt{KT\log(T)}\right),
\end{align*}
where $\epsilon$ is taken as $62M\sqrt{\frac{K\ln(T)}{T}}$ in inequality (a). Theorem~\ref{thm:linear_gapfree} is then proved.
\end{proof}

\section{$(\alpha,\beta)$-Approximation Oracle and Regret}\label{app:alpha_beta}
In this section, we discuss how to extend from exact oracles to $(\alpha,\beta)$-approximation oracles, and the corresponding performance guarantees. With the definition given in Section~\ref{subsec:theory_general}, it is straightforward to use $(\alpha,\beta)$-approximation oracles to replace the original exact oracles in BEACON. To facilitate the discussion, we further assume that this approximation oracle always outputs collision-free matchings, which naturally holds for most of approximate optimization solvers \citep{vazirani2013approximation}.

With an $(\alpha,\beta)$-approximation oracle, as stated in Section~\ref{subsec:theory_general}, a regret bound similar to Theorem~\ref{thm:general} can be obtained regarding the $(\alpha,\beta)$-approximation regret. First, the following notations are redefined and slightly abused to accommodate the $(\alpha,\beta)$-approximation regret:
$\Sc_* = \{S|S\in \Sc, V_{\boldsymbol{\mu},S} \geq \alpha V_{\boldsymbol{\mu},*} \}$: the set of matchings with rewards larger than $\alpha V_{\boldsymbol{\mu},*}$; $\Delta^{k,m}_{\min} = \alpha V_{\boldsymbol{\mu},*} - \max\{V_{\boldsymbol{\mu},S}|S\in\Sc_b, s_m = k \}$; $\Delta^{k,m}_{\max} = \alpha V_{\boldsymbol{\mu},*} - \min\{V_{\boldsymbol{\mu},S}|S\in\Sc_b, s_m = k \}$. With these notations, BEACON's performance with an approximate oracle is established in the following.
\begin{theorem}[$(\alpha,\beta)$-approximation regret]\label{thm:app_regret}
	Under Assumptions~\ref{asp:reward}, \ref{asp:mono}, and  \ref{asp:bound}, with an $(\alpha,\beta)$-approximation oracle, the $(\alpha,\beta)$-approximation regret of BEACON is upper bounded as
	\begin{align*}
	R(T)& =\tilde{O}\left( \sum_{(k,m)\in[K]\times [M]}\left[\frac{\Delta^{k,m}_{\min}}{(f^{-1}(\Delta^{k,m}_{\min}))^2}+ \int_{\Delta^{k,m}_{\min}}^{\Delta^{k,m}_{\max}} \frac{1}{(f^{-1}(x))^2}\mathrm{d}x\right]\log(T)+ M^2K\Delta_c\log(T) \right). 
	\end{align*}
\end{theorem}   
\begin{proof}
The proof for Theorem~\ref{thm:app_regret} closely follows the proof for Theorem~\ref{thm:general}. To avoid unnecessarily redundant exposition, we here only highlight the key steps and major differences.

The communication regret and the other regret can be obtained with the same approach in the proof for Theorem~\ref{thm:general}. The main difference lies in the exploration regret. In the following proof, unless specified explicitly before, the adopted notations share the same definition as in the proof for Theorem~\ref{thm:general}. Similar to Eqn.~\eqref{eqn:regret_decomp_general}, we can decompose the exploration regret w.r.t. the definition of the $(\alpha,\beta)$-approximation regret as
	\begin{align*}
		R_e(T) &= \Eb\left[\sum_r 2^{p_r}(\alpha\beta V_{\boldsymbol{\mu},*}-V_{\boldsymbol{\mu},S_r})\right]\\
		&=\Eb\left[\sum_r 2^{p_r}(\alpha V_{\boldsymbol{\mu},*}-V_{\boldsymbol{\mu},S_r})\right]+\alpha(\beta-1)V_{\boldsymbol{\mu},*} \Eb[T_e]\\
		&=\Eb\left[\sum_r 2^{p_r}(\alpha V_{\boldsymbol{\mu},*}-V_{\boldsymbol{\mu},S_r})(\mathds{1}\{\Gc_r\}+\mathds{1}\{\bar{\Gc}_r\})\right]+\alpha(\beta-1)V_{\boldsymbol{\mu},*} T_e\\
		&\leq \Eb\left[\sum_r 2^{p_r}(\alpha V_{\boldsymbol{\mu},*}-V_{\boldsymbol{\mu},S_r})(\mathds{1}\{\Gc_r\}+1-\beta)\right]+\alpha(\beta-1)V_{\boldsymbol{\mu},*} T_e\\
		&\leq \Eb\left[\sum_r 2^{p_r}(\alpha V_{\boldsymbol{\mu},*}-V_{\boldsymbol{\mu},S_r})\mathds{1}\{\Gc_r\}\right]
	\end{align*}
where $T_e$ is the length of overall exploration phases. Notation $\Gc_r := \{V_{\boldsymbol{\mu},S_r}\geq \alpha V_{\boldsymbol{\mu},*}\}$ denotes the event that the oracle successfully outputs a good matching at epoch $r$, which happens with a probability at least $\beta$. Then, conditioned on event $\Gc_r$, the remaining analysis follows the same process in the proof for Lemma~\ref{lem:expl_regret}, and Theorem~\ref{thm:app_regret} can be obtained.
\end{proof}
\end{document}